\documentclass[ejsv2]{imsart}

\RequirePackage[numbers]{natbib}
\RequirePackage[colorlinks,citecolor=blue,urlcolor=blue]{hyperref}
\RequirePackage{graphicx}
\usepackage{amsmath}
\usepackage{graphicx}
\usepackage{tikz}
\usepackage{pgfplots}
\pgfplotsset{compat=1.18}
\usetikzlibrary{calc, arrows.meta, decorations.pathreplacing}
\usepackage{amsmath}
\usepackage{booktabs}
\usepackage{amsthm} 
\usepackage{tikz,pgfplots}
\pgfplotsset{compat=1.18}
\usepackage{xcolor}
\definecolor{darkgreen}{RGB}{0,100,0}
\usetikzlibrary{positioning,arrows.meta,shapes.geometric,calc}
\usepackage{caption,subcaption}
\usepackage[section]{placeins}
\usepackage{comment}
\startlocaldefs
\theoremstyle{plain}

\newtheorem{theorem}{Theorem}[section]
\newtheorem{lemma}[theorem]{Lemma}
\newtheorem{corollary}[theorem]{Corollary}
\newtheorem{definition}{Definition}[section]
\newtheorem{proposition}[definition]{Proposition}

\newtheorem{remark}{Remark}

\newcommand{\E}{\mathbb{E}}
\theoremstyle{definition}

\endlocaldefs

\begin{document}
\begin{frontmatter}
\title{SPDE Methods for Nonparametric Bayesian Posterior Contraction and Laplace Approximation}

\runtitle{SPDE Methods for Nonparametric Bayesian PCRs and Laplace Approximation}

\begin{aug}
\author[A]{\fnms{Enric}~\snm{Alberola Boloix}\ead[label=e1]{ealberola@bcamath.com}}
\and
\author[B]{\fnms{Ioar}~\snm{Casado Telletxea}\ead[label=e2]{icasado@bcamath.com}}

\address[A]{Basque Center for Applied Mathematics (BCAM), Bilbao (Spain), \printead[presep={\ }]{e1}}
\address[B]{Basque Center for Applied Mathematics (BCAM), Bilbao (Spain), \printead[presep={\ }]{e2}}
\end{aug}

\begin{abstract}
We derive posterior contraction rates (PCRs) and finite-sample Bernstein–von Mises (BvM) results for non-parametric Bayesian models by extending the diffusion-based framework of \cite{mou2024diffusion} to the infinite-dimensional setting. The posterior is represented as the invariant measure of a Langevin stochastic partial differential equation (SPDE) on a separable Hilbert space, which allows us to control posterior moments and obtain non-asymptotic concentration rates in Hilbert norms under various likelihood curvature and regularity conditions. We also establish a quantitative Laplace approximation for the posterior. The theory is illustrated in a nonparametric linear Gaussian inverse problem.
\end{abstract}



\end{frontmatter}
\tableofcontents

\section{Introduction}
With Bayesian procedures now routinely used in high-dimensional and function-space models, a precise understanding of posterior concentration and Gaussian approximation (both asymptotic and finite-sample) is necessary to justify uncertainty quantification beyond heuristic appeals to large 
sample
behavior. The study of large-sample Bayesian asymptotics goes back to Laplace \cite{marquis1820theorie}. In the modern era, \cite{doob1949application} and \cite{schwartz1965bayes} established parametric posterior consistency under conditions that ensure the posterior concentrates in neighborhoods of the data-generating parameter (typically in a topology compatible with identifiability). Consistency alone, however, does not quantify the rate at which mass accumulates near the true parameter. Posterior contraction rates (PCRs) address this by characterizing sequences $r_n\rightarrow 0$ such that the posterior mass outside a ball of radius $r_n$ around the data-generating parameter vanishes with high probability, where $n$ is the sample size. The general testing-and-entropy framework developed in \cite{ghosal2000convergence} and \cite{shen2001rates} unifies parametric and nonparametric analyses and remains the main tool for deriving sharp PCRs.

A complementary refinement is provided by Bernstein–von Mises (BvM) type results \cite{bernstein1917, vonmises1931}. In regular finite-dimensional models, BvM asserts that the posterior, after centering at an efficient estimator (typically the MLE, or an asymptotically equivalent estimator) and scaling by $\sqrt{n}$, converges in Total Variation to a Gaussian distribution with covariance given by the inverse Fisher information. In its modern LAN-based form, this perspective is due to Le Cam \cite{le1953some} (see also Chapter $10$ of \cite{van2000asymptotic}).

In infinite-dimensional models the situation changes qualitatively. Freedman’s counterexample \cite{Freedman} shows that a naive extension of the parametric BvM theorem can fail: the posterior need not be asymptotically Gaussian in strong topologies, and the prior may leave a non-negligible imprint even as 
the sample size goes to infinity. This motivates focusing on PCRs in function spaces under explicit metrics/norms, and on non-asymptotic Gaussian approximation bounds (or BvM-type statements for suitable finite-dimensional projections or smooth functionals) that make the dependence on dimension, regularity, and prior explicit. 

\subsection{Contributions and structure of the paper}

Recently, Mou et al. \cite{mou2024diffusion} obtained PCRs and non-asymptotic BvM theorems by representing the Bayesian posterior as the stationary distribution of a Langevin diffusion process and then controlling the moments of this process using techniques from stochastic calculus. 
While the framework developed by \cite{mou2024diffusion} shows promise as a general method for analyzing the Bayesian posterior, their results are limited to parametric models. In this paper, we extend their diffusion-based approach to the nonparametric setting, thereby giving a positive answer to their suggestion that ``the diffusion process approach [...] might also be fruitfully applied to nonparametric models.'' 
This extension establishes a unified analytical perspective linking frequentist properties of nonparametric Bayesian procedures to infinite-dimensional MCMC algorithms and the theory of infinite-dimensional diffusion models.
\\[4pt]
Extending the results in \cite{mou2024diffusion} to infinite dimensions requires dealing with the theory of stochastic partial differential equations (SPDEs) \cite{da2014stochastic, liu2015stochastic} and Malliavin calculus \cite{NualartMalliavin,malliavin2, malliavin1}, which in turn depend on the theory of Gaussian Hilbert spaces \cite{janson1997gaussian}. In Section \ref{section:background} we introduce the necessary background for our results, including SPDEs and the Cameron-Martin space, which will play a crucial role. In Section \ref{section:pcrs_under_strong} we introduce PCRs in Hilbert norm for non-parametric models under a strong concavity assumption, and Section \ref{weak posterior contraction} weakens this assumption. In Section \ref{section:bvm_laplace} we prove non-asymptotic BvM-style theorems, and our results are illustrated in a simple model in Section \ref{section:example}. All the proofs are contained in Appendix \ref{appendix:proofs}, while the remaining appendices contain auxiliary technical results.

\subsection{Related work}

As mentioned in the introduction, \cite{ghosal2000convergence, shen2001rates} introduced PCRs for parametric and non-parametric models. Since then PCRs have been obtained for specific priors such as Dirichlet and Gaussian process priors \cite{ghosal2007posterior, walker2007rates, van2008rates}. As for BvM theorems, the task is much harder in infinite dimensions, and positive results exists only for specific models such as Gaussian white noise \cite{NICKL_CASTILLO2}, some regression tasks with weak norms \cite{castillo2014bernstein}, semi-parametric models \cite{CASTILLO-ROUSSEAU, PANOV-SPOKOINY}, or models with finite effective dimension \cite{spokoiny}. See \cite{GhosalNonparametric,rousseau2016frequentist} and the references therein for overviews on the frequentist properties of Bayesian non-parametric models. 
\\[4pt]
More recently, \cite{dolera2024strong} obtained PCRs for infinite-dimensional exponential families using Wasserstein dynamics. The study of PCRs and BvM theorems in infinite dimensions also naturally arises in the Bayesian inverse problems literature \cite{StuartActaNumerica, NicklNonlinear, EFFICIENCY_AGAPIOU0, EFFICIENCY_AGAPIOU1, EFFICIENCY_AGAPIOU2,GiordanoHeat, NicklSchrodinger, NicklTimeEvolution, LaplaceHelin}. On a related line, recent work has developed non-asymptotic accuracy guarantees for Laplace (Gaussian) approximations in high dimensions. Katsevich derives a leading-order decomposition of Laplace-approximation error, proving tighter upper bounds and high-dimensional lower bounds in total variation, and showing that, in general, some scaling between sample size and dimension is necessary for accuracy; the same framework yields a skew-adjusted Laplace approximation that provably improves the approximation’s accuracy by one order in the high-dimensional regime, with applications including Dirichlet posteriors from multinomial observations and logistic regression with Gaussian design \cite{LaplaceKatsevich1}.
Complementarily, Spokoiny provides explicit non-asymptotic total-variation bounds for Gaussian/Laplace posterior approximations in terms of an effective dimension that depends on the interplay between likelihood information and prior/penalty strength  \cite{LaplaceSpokoiny1}.
More recently, Katsevich and Spokoiny have developed a framework that unifies several existing Laplace-accuracy results and yields problem-adapted upper bounds that can be substantially tighter than earlier “rigid” bounds; in a prototypical Bayesian inverse problem their optimized bounds are dimension-free and improve prior bounds by an order of magnitude \cite{LaplaceKatsevich2}.
\\[4pt]
The diffusion process perspective on Bayesian inference we exploit in this paper is closely connected to the foundations of many MCMC algorithms, which rely on discretized versions of Langevin dynamics for sampling and analysis \cite{MCMC1, MCMC2, MCMC3, MCMC4, MCMC5}. Parallel to our work, the same tools from SPDEs and Malliavin calculus have been exploited in order to design infinite-dimensional diffusion models \cite{infinite1, infinite2, infinite3, infinite4, infinite5, infinite6, infinite7, infinite8}. See \cite{ReviewInfinite} for a survey on the topic.

\section{Background}\label{section:background}

In this Section we introduce the necessary prerequisites for our main results. We start with an overview of Gaussian measures and Bayesian inference over separable Hilbert spaces, and then turn to the infinite dimensional Langevin diffusion process.

\subsection{Bayesian inference and posterior contraction}
Consider a measurable space,  $(\mathcal{X},\mathcal{F})$,  and a statistical model given by a non-parametric family of probability distributions $\{\mathbb{P}_\theta\,|\, \theta\in\mathcal{H}\}$ over $(\mathcal{X},\mathcal{F})$ indexed by the elements of the separable Hilbert space $(\mathcal{H},\langle\cdot,\cdot\rangle_\mathcal{H})$  equipped with the Borel $\sigma$--algebra $\mathcal B(\mathcal{H})$. Let $X^n$ be a vector of i.i.d. covariates generated from $\mathbb{P}_{\theta^{*}}$, that is, $X^{n}:=(X_1,\ldots,X_n)\sim \mathbb{P}_{\theta^*}^{\otimes n}$, where $\theta^{*}$ is the (unknown) true function to be estimated. We consider a reference centered Gaussian measure $\mu$ with covariance $Q$ defined over $\mathcal{H}$ from which we construct the prior and the posterior measures. That is, we take 
$\mu:=\mathcal N\bigl(0,\,Q\bigr)$,
with $Q:\mathcal{H}\rightarrow\mathcal{H}$ being a  self-adjoint, positive and  trace-class operator with associated Cameron-Martin space 
\begin{align*}
\mathcal H_{Q}:=Q^{1/2}\mathcal{H}=\operatorname{Ran}\bigl({Q}^{1/2}\bigr)
=\Bigl\{h\in \mathcal{H}:\;\|h\|_{{\mathcal H}_{Q}}:=\|Q^{-1/2}h\|_\mathcal{H}<\infty\Bigr\}.
\end{align*}
We will assume that the distributions in $\{\mathbb{P}_\theta\,|\, \theta\in\mathcal{H}\}$ are absolutely continuous with respect to the gaussian measure $\mu$, and that $\theta^{*}\in \mathcal{H}_{Q}$.

In order to perform Bayesian inference, we will consider Gibbs-type priors, which are absolutely continuous with respect to $\mu$, and have densities of the form 
\begin{equation}\label{equation:def_prior}d\pi(\theta):=\frac{e^{-V(\theta)}\,d\mu(\theta)}{\int_{\mathcal{H}}e^{-V(\theta)}\,d\mu(\theta)}=\frac{e^{-V(\theta)}\,d\mu(\theta)}{Z_{0}}, \quad \forall \theta \in \mathcal{H},
\end{equation}
where $Z_{0}<\infty$ is the prior normalizing constant and $V:\mathcal{H}\to\mathbb{R}$ is a deterministic potential. 
Now consider the (empirical) log-likelihood function

\begin{equation}
F_n(\theta):=\frac1n\sum_{i=1}^n \log\frac{d\mathbb P_\theta}{d\mu}(X_{i}),
\end{equation}

which we assume to be twice Fréchet differentiable \cite{malliavin2}. We can now define our Bayesian posterior measure:

\begin{equation}\label{eq:posterior}
d\Pi(\theta\,|\,X^{n} 
):=\frac{e^{nF_n(\theta)-V(\theta)}d\mu(\theta)}{Z}, \quad \forall \theta\in\mathcal{H},
\end{equation}

where $Z<\infty$ is the normalization constant. Under further regularity conditions on $V$ and $F_{n}$, we expect that the posterior will concentrate on smaller neighborhoods of $\theta^*$ as $n$ increases.

\begin{definition}
Given a norm $\|\cdot\|$ on $\mathcal{H}$ and a confidence parameter $\delta\in (0,1)$, we say that $\Pi(\cdot\,|\,X^n)$ concentrates at $\theta^*$ with rate $r(n,\delta)$ if
\begin{align}
\Pi\Big(\|\theta-\theta^*\|\geq r(n,\delta)\,|\,X^n\Big)\leq \delta,
\end{align}
with $\mathbb{P}^{\otimes n}_{\theta^*}$-probability at least $1-\delta$.
\end{definition}
This contraction rate $r(n,\delta)$ will depend on problem-specific parameters beyond $n$ and $\delta$, such as the choice of $Q$ or the smoothness of $F_n$ and $V$. We now introduce the SPDE-based framework for obtaining such rates.

\subsection{Infinite-dimensional Langevin equation}

Following \cite{mou2024diffusion}, we will interpret the Bayesian posterior in (\ref{eq:posterior}) as the invariant measure of certain semilinear SPDE with values in $\mathcal{H}$.
Specifically, consider the following infinite-dimensional version of the preconditioned Langevin equation \cite{MCMC4}:

\begin{equation}\label{Langevin_eq}
\left\{
\begin{aligned}
\mathrm{d}\theta_t
&=\left[-\frac{1}{n}\theta_t -\frac{1}{n}Q\nabla_\mathcal{H}V(\theta_t)+ Q\nabla_{\mathcal H}F_n(\theta_t)\right]\mathrm{d}t
+ \sqrt{\frac{2}{n}}\,\mathrm{d}W_t^{Q},\\
\theta_0 &= \theta^* .
\end{aligned}
\right.
\end{equation}

where $\nabla_\mathcal{H}$ is the Riesz's representer of the Fréchet derivative $\mathcal{D}_\mathcal{H}$ (see Appendix A), and $\left(W^{Q}_t\right)_{t\geq 0}$ is a $Q$-Wiener process on $\mathcal{H}$. The following assumptions ensure the existence and uniqueness of a mild solution to \eqref{Langevin_eq}; see Section~7.1 of \cite{da2014stochastic}.

\begin{description}
\item[(A.1)] \emph{(Lipschitz gradients).}
There exist constants $L_1,L_2>0$ such that, for all $\theta_1,\theta_2\in\mathcal H$,
\begin{align*}
\bigl\| Q\nabla_{\mathcal H}F_n(\theta_1)-Q\nabla_{\mathcal H}F_n(\theta_2) \bigr\|_{\mathcal H}
&\le L_1 \, \|\theta_1-\theta_2\|_{\mathcal H},\\
\bigl\| Q\nabla_{\mathcal H}V(\theta_1)-Q\nabla_{\mathcal H}V(\theta_2) \bigr\|_{\mathcal H}
&\le L_2 \, \|\theta_1-\theta_2\|_{\mathcal H}.
\end{align*}

\item[(A.2)] \emph{(Linear growth).}
There exist constants $C_1,C_2\ge 0$ such that, for all $\theta\in\mathcal H$,
\begin{align*}
\|Q\nabla_{\mathcal H}F_n(\theta)\|_{\mathcal H}^2
&\le C_1\bigl(1+\|\theta\|_{\mathcal H}^2\bigr),\\
\|Q\nabla_{\mathcal H}V(\theta)\|_{\mathcal H}^2
&\le C_2\bigl(1+\|\theta\|_{\mathcal H}^2\bigr).
\end{align*}
\end{description}

Furthermore, under the following monotonicity assumption, the Bayesian posterior (\ref{eq:posterior}) will be the unique stationary distribution of (\ref{Langevin_eq}) \cite{da2014stochastic,MCMC4}:

\begin{description}
  \item[(B)] \emph{(Monotone drift).} For all $\theta_1,\theta_2\in\mathcal{H}$,
    \begin{align*}
     -\frac{1}{n}\langle Q[\nabla_{\mathcal{H}}V(\theta_1)-\nabla_{\mathcal{H}}V(\theta_2)],\theta_1-\theta_2\rangle_{\mathcal{H}} + \langle Q[\nabla_{\mathcal{H}}F_n(\theta_1)-\nabla_{\mathcal{H}}F_n(\theta_2)],\theta_1-\theta_2\rangle_{\mathcal{H}}\leq0.
    \end{align*}
\end{description}

In particular, assumptions \textbf{(A)} and \textbf{(B)}, imply that the laws $\mathcal{L}(\theta_t)$ converge weakly to $\Pi(\theta\mid X_1^n)$ \cite{da2014stochastic}. This fact, along with the following result (which can be considered a generalization of Proposition 2.1 in \cite{mou2024diffusion}), will allow us to control the moments of $\mathcal{L}(\theta_t)$.

\begin{proposition}\label{auxiliary_moment_prop}
    Consider a sequence of probability measures $(\pi_t)_{t\geq0}$ on $\mathcal{H}$ such that $\pi_t\Rightarrow\pi^*$ (weak convergence), and suppose that we have $\sup_{t\geq 0} \mathbb{E}_{\pi_t}\left[\|X\|_\mathcal{H}^r \right]<\infty$ and $\mathbb{E}_{\pi^*}\left[\|X\|_\mathcal{H}^r \right]<\infty$ for any $r\geq 2$. We then have $\lim_{t\rightarrow\infty}\mathbb{E}_{\pi_t}\left[\|X\|_\mathcal{H}^p \right] = \mathbb{E}_{\pi^*}\left[\|X\|_\mathcal{H}^p \right]$ for any $p\geq 2$.
\end{proposition}
Finally, for the computations in the proof of the Laplace Approximation, we will assume that $\nabla_{\mathcal{H}}V,\nabla_{\mathcal{H}}F\in L^2(\mu;\mathcal{H}_{Q})$ (see Appendix $B$).

\section{Posterior contraction rates (PCRs)}

We now obtain PCRs for Bayesian nonparametric models based on the SPDE representation of the posterior in (\ref{Langevin_eq}). The proof strategy mirrors that in \cite{mou2024diffusion}, with the added technicalities derived from dealing with infinite-dimensional processes. In what follows, the PCRs we obtain will depend on curvature assumptions on the population log–likelihood:
\begin{equation*}
  F(\theta):=\mathbb E_{X\sim \mathbb{P}_{\theta^{\!*}}}\left[\log\frac{d\mathbb P_\theta}{d\mu}(X)\right].
\end{equation*}

\subsection{PCRs under strong concavity}\label{section:pcrs_under_strong}

We start by assuming $F$ is strongly concave at $\theta^*$, and that there is uniform control of the gradients.

\begin{description}
  \item[(C.1)] \emph{(One Point Strong concavity at $\theta^{*}$)} There exists $\mu>0$ such that
  \[
    -\langle Q\nabla_\mathcal{H} F(\theta),\theta^{\!*}-\theta\rangle_\mathcal{H}\;\ge\;\mu\,\|\theta-\theta^{\!*}\|_\mathcal{H}^{2},
    \quad\forall\,\theta\in \mathcal{H}.
  \]
  \item[(C.2)] \emph{(Uniform gradient control)} For any $r>0$ and $\delta\in(0,1)$,
  \[
    \mathbb{P}_{\theta^{*}}^{\otimes n}\left(\sup_{\theta\in\mathbb B_\mathcal{H}(\theta^{\!*},r)}\|Q\nabla_\mathcal{H} F_n(\theta)-Q\nabla_\mathcal{H} F(\theta)\|_\mathcal{H}
    \;\le\;\varepsilon_1(n,\delta)\,r+\varepsilon_2(n,\delta)\right)
    \geq 1-\delta.
  \]
\end{description}
Assumption \textbf{(C.2)} can be verified with methods from empirical process theory. In order to simplify our proofs, we also assume that we have the following one-sided prior control:

\begin{description}
    \item[(C.3)] \emph{(One-sided prior control)}
    There exists a positive constant $B>0$ such that
\begin{align*}
    -\langle Q\nabla_{\mathcal{H}} V(\theta),\theta-\theta^*\rangle_\mathcal{H}\leq B\|\theta-\theta^*\|_{\mathcal{H}}, \quad \forall \theta\in\mathcal{H}.
\end{align*}
\end{description}
Observe that Assumption \textbf{(C.3)} is mild and holds for cases such as $V=0$ or $V\propto\|\theta\|^\alpha_\mathcal{H}$ for any $\alpha>1$. With these assumptions, we can already state our main PCR result:

\begin{theorem}\label{mainthm:posterior_contraction}
    Under Assumptions \textbf{(A)}, \textbf{(B)} and \textbf{(C)}, there is a universal constant $c>0$ such that, for every $\delta\in(0,1)$ and all $n$ sufficiently large so that $\varepsilon_1(n,\delta)\le\mu/6$, it holds 
\begin{align}
  \mathbb{P}^{\otimes n}_{\theta^{*}}\Bigg(\Pi\biggl(\|\theta-\theta^{\!*}\|_\mathcal{H}\geq c\sqrt{\tfrac{tr(Q)}{n\mu}} +\tfrac{B}{n\mu} +\tfrac{\varepsilon_2(n,\delta)}{\mu} +c\sqrt{\tfrac{\|Q\|_{\text{op}}\log(1/\delta)}{n\mu}}\;\Bigm| \; X_1^n\biggr)\;\le\;\delta\Bigg) \geq 1 -\delta.
\end{align}
\end{theorem}

Observe that we obtain rates of order $\mathcal{O}\left(\frac{1}{\sqrt{n}}\right)$, which are not common in infinite dimensions. This is due to the strength assumption \textbf{(C.1)}, which is usually false for many nonparametric inverse problems that come from Partial Differential Equations (see, e.g., \cite{NicklNonlinear}).

The next section deals with PCRs under a weak concavity assumption. The tools involved do not change substantially from those of \cite{mou2024diffusion}. 

\subsection{PCRs under weak concavity}
Let's assume that the following conditions hold:

\begin{description}
\item[(W.1)]\emph{(Weak Concavity)} There exists a convex, non-decreasing function $\psi:[0,+\infty) \rightarrow \mathbb{R}$ such that
$$
-\left\langle Q\nabla_\mathcal{H}F(\theta), \theta-\theta^{*}\right\rangle_{\mathcal{H}} \geq \psi\left(\left\|\theta-\theta^{*}\right\|_\mathcal{H}\right) \ \forall \theta \in \mathcal{H}.
$$

\end{description}

\begin{description}
\item[(W.2)] \emph{(Empirical Process Bound)} There exist a function $\varepsilon: \mathbb{N} \times(0,1] \mapsto \mathbb{R}^{+}$and a non-decreasing function $\zeta: \mathbb{R} \rightarrow \mathbb{R}$ such that $\zeta(0) \geq 0$ and 
$$
\mathbb{P}_{\theta^{*}}\Bigg(\sup _{\theta \in \mathbb{B}\left(\theta^{*}, r\right)}\left\|Q\nabla F_{n}(\theta)-Q\nabla F(\theta)\right\|_H \leq \varepsilon(n, \delta) \zeta(r)\Bigg) \geq 1- \delta \ \ \forall r >0.  
$$
\end{description}

\begin{description}
    \item[(W.3)]\emph{(Diffusion Moments Control)} 
Let $\xi: \mathbb{R}_{+} \rightarrow \mathbb{R}$ be the inverse of $r \mapsto r \zeta(r)$. Then, assume that the function $r \mapsto \psi(\xi(r))$ is convex, and $\psi$ and $\zeta$ satisfy the differential inequalities
\begin{align*}
&r \psi^{\prime}(r) \zeta(r) \geq r \psi(r) \zeta^{\prime}(r)+\psi(r) \zeta(r), 
\\
\\
&r^{2} \psi^{\prime \prime}(r) \zeta(r)+r \psi^{\prime}(r) \zeta(r) \geq 3 \psi(r) \zeta(r)+r^{2} \psi(r) \zeta^{\prime \prime}(r) \ \forall r>0.
\end{align*}

\end{description}

\begin{description}
    \item[(W.4)]\emph{(Existence of Unique Positive Solution to an auxiliary Fixed Point Equation)} 
Consider the following fixed point equation in the variable $z>0$ for each fixed tolerance $\delta \in(0,1)$ and sample size $n$:
\begin{equation}\label{eq:fixed_point}
\psi(z)=\varepsilon(n, \delta) \zeta(z) z+\frac{B}{n} z+\frac{\text{tr}(Q)}{n}+\frac{\log (1 / \delta)\|Q\|_{\text{op}}}{n}
\end{equation}
Then, assume that the limit $$\lim \inf _{z \rightarrow+\infty} \frac{\psi(z)}{z \zeta(z)}$$ is strictly positive, and the sample size $n$ and tolerance parameter $\delta \in(0,1)$ are such that $\varepsilon(n, \delta)<\lim \inf _{z \rightarrow+\infty} \frac{\psi(z)}{z \zeta(z)}$.
\end{description}
Note that conditions \textbf{(W.1)} and \textbf{(W.2)} are generalizations of assumptions \textbf{(C.1)} and \textbf{(C.2)}, respectively. Assumptions \textbf{(W.3)} and \textbf{(W.4)} are technical conditions analogous to those in \cite{mou2024diffusion}. With this set-up, we are now ready to state our second main result:
\\
\begin{theorem}\label{weak posterior contraction} Suppose that Assumptions \textbf{(A)}, \textbf{(B)}, and \textbf{(W.1)-(W.3)} hold. Then for any given sample size $n$ and $\delta \in(0,1)$ such that Assumption \textbf{(W.4)} holds, equation (\ref{eq:fixed_point}) has a unique positive solution $z^{*}(n, \delta)$ such that
\begin{equation}
\mathbb{P}^{\otimes n}_{\theta^{*}}\Big(\Pi\left(\left\|\theta-\theta^{*}\right\|_\mathcal{H} \geq z^{*}(n, \delta) \mid X^{n}\right) \leq \delta\Big) \geq 1-\delta.
\end{equation}
\end{theorem}

\section{Laplace-type Approximation}\label{section:bvm_laplace}
In this section we specify the regularity conditions for $V, F_{n}$ and $F$ under which the posterior distribution converges to a Gaussian distribution centered at an efficient (in the sense defined below) estimator of the true function $\theta^{*}$, and whose variance depends on the data and on the prior. Following \cite{mou2024diffusion}, we consider the infinite dimensional version of the Maximum  a Posteriori (MAP) estimator, that we denote  by $\widehat\theta^{(n)}$ for a given sample size $n$.  That is, given the Onsager-Machlup (OM) functional (an object that describes the asymptotic ratio of small ball probabilities centered at different points of $\mathcal{H}_{Q}$)
\begin{align*}
\mathcal{J}_{n}:&\mathcal{H}_{Q}\longrightarrow \mathbb{R} \\
&\theta\longmapsto nF_{n}(\theta)+V(\theta) + \frac{1}{2}\|\theta\|_{\mathcal{H}_{Q}}^{2},
\end{align*}
we will have 
$\widehat{\theta}^{(n)}\in\arg  \min_{\theta \in \mathcal{H}_Q}  \mathcal{J}_{n}(\theta)$
.

\begin{remark}
It is worth noting that there have been several proposals for the definition of a MAP estimator in infinite-dimensional linear spaces. For an infinite-dimensional Hilbert space equipped with a Gaussian Prior, Dashti et al. \citep{MAPstuart} provide the equivalence between the variational characterization of the MAP as the Minimizer of the Onsager-Machlup functional  and its topological characterization as the center of a small ball with maximum a posteriori probability (that is, as a posterior mode). On the other hand, \cite{MAPgauss} and \cite{MAPinf} introduce, respectively, the concepts of strong and weak MAPs, and prove the equivalence between these two notions for Gaussian priors over Hilbert spaces. In this setting, \cite{onsager1}  proved that these two concepts coincide for Bayesian Inverse problems with Gaussian (and Laplacian) noise and a Lipschitz continuous log-likelihood. Recently, \cite{onsager2} and \cite{onsager3} have considered MAPs for priors beyond Gaussian ones, like Besov or Cauchy priors. We will restrict to the interpretation of the MAP as a penalized-likelihood estimator with a squared Cameron Martin (Reproducing Kernel Hilbert Space)-norm penalty (see \cite{GhosalNonparametric} and \cite{Giné_Nickl_2015}).
\end{remark}
Now, since $\widehat{\theta}^{(n)}$ is a minimizer of the OM functional, we will have 
\begin{align*}
D_{|\mathcal{H}_Q}\mathcal{J}_{n}(\widehat{\theta}^{(n)})
= nD_{|\mathcal{H}_Q}F_{n}(\widehat{\theta}^{(n)})
+ & D_{|\mathcal{H}_Q}V(\widehat{\theta}^{(n)})
+ Q^{-1}\big[\widehat{\theta}^{(n)}\big]
 = 0
\end{align*}
where $D_{|\mathcal{H}_{Q}}$ is the restriction to $\mathcal{H_{Q}}$ of the Fréchet derivative and (in the domain of $Q^{-1}$)  $$Q^{-1}[\widehat{\theta}^{(n)}](h):=\langle Q^{-1}\widehat{\theta}^{(n)},h \rangle_\mathcal{H}=\langle Q^{-1/2}\widehat{\theta}^{(n)},Q^{-1/2}h \rangle_\mathcal{H}=\langle\widehat{\theta}^{(n)},h \rangle_{\mathcal{H}_{Q}}.$$
Denoting $\nabla_{\mathcal{H}_{Q}}\mathcal{J}_{n}(\widehat{\theta}^{(n)})$ the Riesz representer of $D_{|\mathcal{H}_{Q}}\mathcal{J}_{n}(\widehat{\theta}^{(n)})$ with respect to the Cameron-Martin inner product,  $\langle \cdot, \cdot \rangle_{\mathcal{H_{Q}}}$, we will have (in $\mathcal{H}_{Q}$)
\begin{align}
\nabla_{\mathcal{H}_{Q}}\mathcal{J}_{n}(\widehat{\theta}^{(n)})=Q\nabla_{\mathcal{H}}\mathcal{J}_{n}(\widehat{\theta}^{(n)}),
\end{align}
see also \cite{malliavin1, malliavin2}. We will use this fact in the proof of Theorem \ref{bernstein von mises}. 
Now, let's denote by $H^{*}:=\nabla_{\mathcal{H}}^{2}F(\theta^{*})$ the Fisher Information. 
Then, the next theorem guarantees that the posterior will converge, in Kullback-Leibler divergence, to the Gaussian law given by 
\begin{align}
\gamma_{\hat{\theta}^{(n)}}:=\mathcal{N}\left(\widehat{\theta}^{(n)},\left(Q^{-1} + n H^{*}\right)^{-1}\right).
\end{align}
More concretely, we prove that,  with high $\mathbb{P}^{n}_{\theta^{*}}$-probability, the 
Kullback-Leibler divergence $D_{KL}$ between the posterior and $\gamma_{\hat{\theta}^{(n)}}$, 
\begin{align}
D_{KL}(\Pi(\cdot|X^{n})||\gamma_{\hat{\theta}^{(n)}}) = \int_H\log\left(\frac{d\Pi}{d\gamma_{\hat{\theta}^{(n)}}}(\theta|X^{n})\right)d\Pi(\theta|X^{n}),
\end{align}
will decrease as $\mathcal{O}(1/n^{\beta})$, for some positive $\beta$ that may depend on the regularity of the forward operator, the prior, the dimension of the observations...
 This can be viewed as a quantitative Laplace approximation in the
infinite-dimensional setting. The same covariance operator
have been studied, for instance, in \cite{StuartActaNumerica} for models with
finite-dimensional observations $Y_i=\mathcal{G}(X_i)+\varepsilon_i$,
$i=1,\dots,n$, and in \cite{Covarianza1,Covarianza2,Covarianza3,Covarianza4,Covarianza5}.

\subsection{Necessary conditions for measure comparison}
We now specify the necessary conditions that make $\gamma_{\widehat{\theta}^{(n)}}$ a bona fide Gaussian measure over $\mathcal{H}$ and that allow us to compare it to the reference Gaussian. For the  following two theorems, we will assume that $Q$ has a discrete spectrum with eigenvalues in decreasing order $\lbrace \mu_{m}\rbrace_{m \in \mathbb{N}}$ and that $H^{*}$ also has a discrete spectrum $\lbrace \lambda_{m}\rbrace_{m \in \mathbb{N}}$ and the same eigen-basis $\lbrace e_{m} \rbrace_{m \in \mathbb{N}}$ as $Q$. This way, it is clear that 
\begin{align}
\left(n H^{*}+ Q^{-1}\right)^{-1}e_{m}= \dfrac{\mu_{m}}{n\mu_{m}\lambda_{m} + 1}e_{m} \quad \forall m \in\mathbb{N}.
\end{align}
The fact that this operator is a bona-fide covariance operator follows from $H^{*}$ being positive-definite and bounded and $Q$ being trace-class. Indeed, since $\left(n H^{*}+ Q^{-1}\right)^{-1}$ must be a (non-degenerate) covariance operator, it must be positive-definite, that is 
\begin{align}\label{Covariance_PD}
&\lambda_{m}\mu_{m} > -\dfrac{1}{n}, \quad \forall m \in \mathbb{N},
\end{align}
and trace-class, that is 
\begin{align}\label{Covariance_TraceClass}
\sum_{m\in\mathbb{N}}\dfrac{\mu_{m}}{n\mu_{m}\lambda_{m} +  1} < \infty.
\end{align}
For each $n \in \mathbb{N}$, (\ref{Covariance_PD}) is verified by the positive-definiteness of $QH^{*}$ and, since $1+n\lambda_{m}\mu_{m}\geq 1$ for each $m \in \mathbb{N}$,  (\ref{Covariance_TraceClass}) also holds. Indeed, taking the supremum in $m$  of $1/(1+n\lambda_{m}\mu_{m})$, we will have 
\begin{align}
\sum_{m\in\mathbb{N}}\dfrac{\mu_{m}}{n\mu_{m}\lambda_{m} +  1} < \left(\underset{m\in\mathbb{N}}{\sup}\frac{1}{1+n\lambda_{m}\mu_{m}}\right)\sum_{m\in\mathbb{N}}\mu_{m} \leq \operatorname{tr}(Q)< \infty.  
\end{align}
Now, since, by construction, the posterior is equivalent to the reference Gaussian measure $\mu$, it is clear that imposing $\gamma$ to be equivalent to $\mu$ will guarantee that the Kullback-Leibler divergence above is finite. It is here where 
Feldman-Hájek theorem (see Appendix B) provides the sufficient conditions on $nH^{*}$ for the equivalence between measures to hold: 
\begin{description}  
\item[(F-H.1)]\emph{(Compact Perturbation of the Identity)}
\begin{align}
Q^{-1/2}(Q^{-1}+nH^{*})^{-1}Q^{-1/2}-I \quad \text{is Hilbert-Schmidt on} \quad \mathcal{H},
\end{align}
\item[(F-H.2)]\emph{(Equivalent Cameron-Martin Spaces)}
\begin{align}
(nH^{*})^{-1/2}\mathcal{H}:= \mathcal{H}_{(nH^{*})^{-1}}=\mathcal{H}_{Q}.
\end{align}
\end{description}  
Here assumption \textbf{(F-H.1)} imposes 
\begin{align}
\sum_{m=1}^{\infty} 
\left(
\frac{1}{1 + n\,\lambda_m\,\mu_m} - 1
\right)^{2}
=
\sum_{m=1}^{\infty}
\left(
\frac{n\,\lambda_m\,\mu_m}{1 + n\,\lambda_m\,\mu_m}
\right)^{2}
< \infty,
\end{align}
which, for each finite sample size $n \in \mathbb{N}$, is satisfied if  $\lbrace \lambda_{m}\mu_{m}\rbrace_{m\in\mathbb{N}}\in \ell^{2}(\mathbb{R})$, and  since $\ell^{1}\subset \ell^{2}$ (that is, a trace-class operator must also be Hilbert-Schmidt), it holds particularly for bounded $H^{*}$. On the other hand, assumption \textbf{(F-H.2)} imposes the same decay rate, up to constants, of the eigenvalues of both covariance operators, which is translated in the existence of constants $0<c\leq C < \infty$ that control uniformly the ratio of eigenvalues, that is
\begin{align*}
c \leq \frac{1}{\mu_{m}}\dfrac{\mu_{m}}{n\mu_{m}\lambda_{m} +  1} \leq C \quad \forall m \in \mathbb{N}.   
\end{align*}
Now, by the Lipschitz property of $Q\nabla_{\mathcal{H}}F$ (see \textbf{(B.1)}) and the differentiabiliy properties of $F$, the operator norm of $QD^{2}_{\mathcal{H}}F(\theta^{*})$ must be upper-bounded by $L$, so $\mu_{m}\lambda_{m}\leq L, $ for all $m>0$, so it is enough to take 
\begin{align*}
 C=1, \quad c=\frac{1}{nL+1}   
\end{align*}
to ensure that (\textbf{F-H.2}) holds.

\subsection{Smoothness and Empirical Process bounds on $H^{*}$}
Now consider the following set of assumptions:
\begin{description}
\item[(BvM.1)]\emph{(Smoothness)} $\exists A>0$ such that the population log-likelihood function $F$ satisfies
$$\left\|Q^{1/2}\nabla^{2}_\mathcal{H} F(\theta)-Q^{1/2}\nabla^{2}_\mathcal{H} F\left(\theta^{*}\right)\right\|_{\text{op}} \leq A\left\|\theta-\theta^{*}\right\|_{\mathcal{H}} \quad \forall \theta \in \mathcal{H},
$$
\item[(BvM.2)]\emph{(Empirical Process bound)} For any $\delta,r >0$, there exist non-negative functions $\varepsilon_{1}^{(2)}$ and $\varepsilon_{2}^{(2)}$ with domain $\mathbb{N} \times(0,1]$ and order $\mathcal{O}(1/\sqrt{n})$ such that
\begin{align*}
\mathbb{P}^{\otimes n}_{\theta^{*}}\Bigg(\sup _{\theta \in \mathbb{B}\left(\theta^{*}, r\right)}\left\|Q^{1/2}\nabla^{2}_\mathcal{H} F_{n}(\theta)-Q^{1/2}\nabla^{2}_\mathcal{H}F(\theta)\right\|_{\text{op}} \leq \varepsilon_{1}^{(2)}(n, \delta) r+\varepsilon_{2}^{(2)}(n, \delta)\Bigg) \geq 1-\delta.
\end{align*}
\end{description}
Here the norm $\|\cdot\|_{\text{op}}$ is the operator norm of (bounded) linear endomorphisms on $\mathcal{H}$. These conditions, needed to connect the shapes of $H^{*}$ and of the posterior, are the infinite-dimensional analogues of conditions \textbf{(BvM.1)} and \textbf{(BvM.2)} in \cite{mou2024diffusion}, now in operator form and adapted to the Cameron-Martin geometry. 
\\ 
\subsection{Consistency assumptions}
Finally, to obtain a Laplace-type limit (in the sense specified
below), we assume both that posterior contraction holds and that the MAP
estimator is consistent at a polynomial rate $n^{-\alpha}$, where $\alpha \in (\frac{1}{4},\frac{1}{2}]$. In nonparametric inverse problems, the parametric rate $\alpha=1/2$ is
typically only available under strong regularity (often for sufficiently
regular functionals, see \cite{VANDERWAART_PCR}), whereas the MAP/posterior mean usually converges at
slower, problem-dependent rates. For linear Gaussian inverse problems with
polynomially decaying singular values (mildly ill-posed), the nonparametric
rate is $n^{-s/(2s+2\kappa+d)}$ (see \cite{KnapikVdVZanten2011BIP,EFFICIENCY_AGAPIOU0}),
where $s$  denotes the Sobolev smoothness of $\theta^\ast$, $\kappa$ the
ill-posedness index, and $d$ the spatial dimension; for exponentially decaying
singular values (severely ill-posed inverse problems, like the one for the Heat Equation) one obtains logarithmic rates
\cite{EFFICIENCY_AGAPIOU1,EFFICIENCY_AGAPIOU2}. In direct nonparametric
regression with Gaussian process priors (conjugate Gaussian models), the MAP
coincides with the posterior mean and converges at order $n^{-\alpha/(2\alpha+d)}$
(up to logarithmic factors) \cite{van2008rates,Giné_Nickl_2015}. Related MAP
rates for PDE-constrained regression are given in \cite{NicklMAP1,EFFICIENCY_SIEBEL}.
\\ \\
Given this brief account of the consistency issue, we simply fix $\alpha\in(1/4,1/2]$ and assume that, for each
$\delta>0$, there exists a constant $\sigma(\delta)>0$ such that the following
conditions hold:
 
\begin{description}
\item[(E.1)]\emph{(Rate of convergence of the MAP estimator)} 
\begin{align}
\mathbb{P}^{\otimes n}_{\theta^{*}}\left(\left\|\widehat{\theta}^{(n)}-\theta^{*}\right\|_{\mathcal{H}} \leq  \dfrac{\sigma}{n^{\alpha}}\right) \geq 1-\delta,
\end{align}\\
\item[(E.2)]\emph{(Posterior Contraction Rates)} 
\begin{align}\label{PCR for BvM}
\mathbb{P}^{\otimes n}_{\theta^{*}}\Bigg(\mathbb{E}_{\Pi(\cdot| X^{n})}\Big[\left\|\theta-\theta^{*}\right\|_{\mathcal{H}}\Big]\leq  \frac{\sigma}{n^{\alpha}}\Bigg) \geq 1-\delta.   
\end{align} 
\end{description}
\subsection{Laplace Approximation}
Having presented the above assumptions we will use, we have the following theorem: 
\begin{theorem}\label{bernstein von mises}\textbf{(Finite Sample Laplace-type Approximation)}
Let \textbf{(B)}, \textbf{(BvM.1)}, \textbf{(BvM.2)}, \textbf{(E.1)}, \textbf{(E.2)} and \textbf{(F-H.1)} hold.
Then, for each $\delta \in (0,1)$, there exist positive constants $c_{1}$ and $c_{2}$ such that
\begin{align}\label{InexactLaplace}
\mathbb{P}^{\otimes n}_{\theta^{*}}\Bigg(D_{K L}\left(\Pi\left(\cdot \mid X^{n}\right) \| \gamma_{\widehat{\theta}^{(n)}}\right) \leq \mathcal{H}(n,\alpha,\delta)\Bigg) \geq 1-2 \delta,
\end{align}
where 
\begin{align}\label{LaplaceBound}
\mathcal{H}(n,\alpha,\delta):=\dfrac{c_{1}(A^{2}+ \varepsilon_{1}^{(2)}(n,\delta)^{2})}{\lambda_{\min}(H^{*})}n^{1-4\alpha} + \dfrac{c_{2}}{\lambda_{\min}(H^{*})}\left[\varepsilon_{2}^{(2)}(n,\delta)^{2}+\dfrac{\|Q\|_{\text{op}}L_{2}^{2}}{n^{2}}\right]n^{1-2\alpha}.
\end{align}
\end{theorem}
\begin{remark}
Note that, since (by assumption) the true parameter $\theta^{*}$ and (by the definition of Onsager-Machlup functional) the MAP estimator $\widehat{\theta}^{(n)}$ belong to the Cameron-Martin Space $\mathcal{H}_{Q}$, the $\|\cdot\|_{\mathcal{H}}$ norm in condition \textbf{(E.1)} can be substituted by the $\|\cdot\|_{\mathcal{H}_{Q}}$ one, up to multiplicative constant. Note also that we have assumed, for simplicity, that the dependency of $\delta$ is carried by $\sigma$, and not by $n$, but one could modify this condition as is done in (\textbf{C.2}) and (\textbf{BvM.2}).
\end{remark}

\begin{remark}
Something similar can be said about \textbf{(BvM.1)}. If one removes the preconditioning $Q$ in these conditions, then the spectral norm of the covariance operator, $\|Q\|_{\text{op}}$ appears as a constant multiplying $\mathcal{H}$ in (\ref{InexactLaplace}).
\end{remark}

As a consequence of this theorem, we obtain the nonparametric version of the Laplace Approximation: 
\begin{theorem}\textbf{(Laplace Approximation)}\label{theorem:laplace-approx}
Assume that the Hessian of the Empirical Likelihood evaluated at the MAP estimator, $\widehat{H}_{n}:=-\nabla_{\mathcal{H}}^{2}F_{n}(\widehat{\theta}^{(n)})$, is bounded and that has minimum eigenvalue $\lambda_{\min}(\widehat{H}^{n})$. Let 
$$\widehat{\gamma}_{\widehat{\theta}^{(n)}}:= \mathcal{N}\left(\widehat{\theta}^{(n)},\left(Q^{-1} + n\widehat{H}_{n}\right)^{-1}\right).$$ Then, if Theorem \ref{bernstein von mises} holds, we will have
\begin{align*}
\mathbb{P}^{\otimes n}_{\theta^{*}}\Bigg(D_{K L}\left(\Pi\left(\cdot \mid X^{n}\right) \| \widehat{\gamma}_{\widehat{\theta}^{(n)}}\right) \leq \mathcal{K}(n,\alpha,\delta)+\mathcal{O}(1/n)\Bigg) \geq 1-2 \delta,
\end{align*}
where $\mathcal{K}(n,\alpha,\delta)$ is 
\begin{align*}
&\dfrac{(c_{1}\operatorname{tr}(Q)^{2}+4\sigma^{4})A^{2}}{\lambda_{\min}(\widehat{H}^{(n)})}n^{1-4\alpha} +\dfrac{(c_{2}\operatorname{tr}(Q)+4\sigma^{2})}{\lambda_{\min}(\widehat{H}^{(n)})}\varepsilon_{2}^{(2)}(n,\delta)^{2}n^{1-2\alpha}
\end{align*}
\end{theorem}

\section{Example: Linear Gaussian model}\label{section:example}

Let's consider two separable Hilbert Spaces $(\mathcal H,\langle\cdot,\cdot\rangle_{\mathcal H})$ and $(\mathcal Y,\langle\cdot,\cdot\rangle_{\mathcal Y})$
and $G$ a non-compact bounded linear operator between these spaces, $G:\mathcal H\to\mathcal Y$. 
Let $(\mathcal H,\langle\cdot,\cdot\rangle_{\mathcal H})$ and
$(\mathcal Y,\langle\cdot,\cdot\rangle_{\mathcal Y})$ be separable Hilbert spaces,
let $G:\mathcal H\to\mathcal Y$ be bounded linear, and let
$\Gamma:\mathcal Y\to\mathcal Y$ be trace-class, self-adjoint, and strictly positive.
Suppose that
$$
X_i = G\theta^\ast + \varepsilon_i,\qquad i=1,\dots,n,
$$
where $\varepsilon_1,\dots,\varepsilon_n \stackrel{\mathrm{iid}}{\sim}\mathcal N(0,\Gamma)$ in $\mathcal Y$.
Then the sample mean
$$
\bar X_n := \frac1n\sum_{i=1}^n X_i
$$
satisfies
$$
\bar X_n \sim \mathcal N\!\left(G\theta^\ast,\frac1n\Gamma\right).
$$
Equivalently, defining
$
\xi_n := \sqrt n\,(\bar X_n-G\theta^\ast),
$
we have $\xi_n\sim\mathcal N(0,\Gamma)$ in $\mathcal Y$, and hence
$$
\bar X_n = G\theta^\ast + \frac1{\sqrt n}\xi_n.
$$
Since the covariance does not depend on $\theta^\ast$, $\bar X_n$ is sufficient for
$\theta^\ast$, so the experiment generated by $(X_1,\dots,X_n)$ is exactly equivalent
to the Gaussian shift experiment generated by $\bar X_n$. Here the empirical and population log-likelihood are (up to additive constants) 
\begin{equation}\label{eq:Fn}
F(\theta):=-\frac{1}{2}\bigl\|\Gamma^{-1/2}G(\theta-\theta^\ast)\bigr\|_{\mathcal Y}^2, \qquad F_n(\theta) := -\frac{1}{2}\bigl\|\Gamma^{-1/2}(G\theta-\bar X_n)\bigr\|_{\mathcal Y}^2,
\end{equation}
This allows us to define the information operator as 
\begin{equation}\label{eq:Adef}
A := G^\ast \Gamma^{-1}G:\mathcal H\to\mathcal H,
\end{equation}
that will be usually unbounded in its domain since $\Gamma$ is trace-class and $G$ is non-compact. Now, we consider a Gaussian measure over $\mathcal H$ with covariance $Q$ and take it as the prior (that is, $V=0$).  Let's assume that $A$ and $Q$ share the same eigenbasis, with eigenvalues $\lbrace \lambda_{n}\rbrace_{n\in \mathbb{N}}$ and $\lbrace \mu_{n}\rbrace_{n\in \mathbb{N}}$, and the following conditions hold: 
\begin{description}
     \item[(\textbf{L.0})] $QG^{*}\Gamma^{-1}:\mathcal{Y}\rightarrow \mathcal{H}$ extends to a bounded linear operator.\item[(\textbf{L.1})] $QA$ is $c$-coercive (that is  $QA\succeq cI$, \ $c>0$) and bounded.
     
    \item[(\textbf{L.2})] $QAQ$ is a trace-class operator, that is $\sum_{n\in \mathbb{N}}\lambda_{n}\mu_{n}^{2}<\infty$.
\end{description}
Then, we have this result:

\begin{corollary}\label{example}
Assume that conditions (\textbf{L.1}) and (\textbf{L.2}) hold for the White Noise model under consideration. Then, Theorem \ref{mainthm:posterior_contraction} holds with 
\begin{align*}
&\varepsilon_2(n,\delta)
=
\frac{1}{\sqrt n}\left(
\sqrt{\mathrm{tr}(QAQ)} + \sqrt{2\|QAQ\|_{\mathrm{op}}\log(1/\delta)}
\right), \quad 
L_{1}=\|QA\|_{\text{op}}, 
\\ & L_{2}=0, \quad B=1, \quad \varepsilon_{1}(n,\delta):= 0, \quad \mu=c, 
\end{align*}
\end{corollary}

\section{Conclusions and future work}

We introduced an SPDE-based framework for deriving PCRs, finite-sample BvM results, and quantitative Laplace approximations for nonparametric Bayesian models. By representing the posterior as the invariant measure of a Langevin SPDE on a Hilbert space, we extended the diffusion process approach of \cite{mou2024diffusion} to infinite-dimensional settings under well-posedness and curvature assumptions. It is worth noting that, for many PDE models for which PCR and Bernstein-von-Mises have been recently derived, like for the Calderón Problem (\cite{AbrahamCalderon,BohrCalderon}), the Wicksell Problem (\cite{vanderwaartWICKSELL}), the Heat Equation (\cite{GiordanoHeat}),  Schrödinger Equation  (\cite{NicklSchrodinger}), Navier-Stokes Equation (\cite{NicklNavierStokes1,NicklNavierStokes2}) and  Nonlinear Time Evolution Equations (\cite{NicklTimeEvolution}), to name a few, the results in this work do not apply, since the compactness of the solution operator violates the restrictions of strong and weak convexity \textbf{(C.1)} and \textbf{(W.1)}, respectively. In order to obtain PCRs with stochastic calculus techniques one should relax the assumptions to take into account 
the interplay between the sample size and the rate of decay of the eigenvalues of the Fisher Information, in a similar fashion as done in \cite{van2008rates}, for example. We leave this endeavor for future work. 
\bibliographystyle{imsart-number}
\bibliography{main}

\appendix
\newpage
\section{Proofs}\label{appendix:proofs}

\subsection*{Proof of Proposition \ref{auxiliary_moment_prop}}
\begin{proof}
The proof is standard. Fix $p\ge 2$ and choose any $q>p$ (e.g.\ $q=p+1$). By assumption,
\[
\sup_{t\ge 0}\mathbb{E}_{\pi_t}\|X\|_{\mathcal H}^{q}<\infty
\quad\text{and}\quad
\mathbb{E}_{\pi^*}\|X\|_{\mathcal H}^{q}<\infty .
\]
Now, for any $M>0$,
\[
\mathbb{E}_{\pi_t}\!\left[\|X\|_{\mathcal H}^{p}\mathbf 1_{\{\|X\|_{\mathcal H}^{p}>M\}}\right]
= \mathbb{E}_{\pi_t}\!\left[\|X\|_{\mathcal H}^{p}\mathbf 1_{\{\|X\|_{\mathcal H}>M^{1/p}\}}\right]
\le M^{-(q-p)/p}\mathbb{E}_{\pi_t}\|X\|_{\mathcal H}^{q},
\]
hence $\{\|X\|_{\mathcal H}^{p}:X\sim\pi_t\}$ is uniformly integrable.

Let $f_M(x)=\|x\|_{\mathcal H}^{p}\wedge M$. Since $f_M$ is bounded and continuous and $\pi_t\Rightarrow\pi^*$,
\[
\mathbb{E}_{\pi_t}f_M(X)\underset{t\to\infty}{\longrightarrow} \mathbb{E}_{\pi^*}f_M(X).
\]
Now uniform integrability gives
\[
\mathbb{E}_{\pi_t}\|X\|_{\mathcal H}^{p}-\mathbb{E}_{\pi_t}f_M(X)\underset{M\to\infty}{\longrightarrow} 0
\quad\text{and}\quad
\mathbb{E}_{\pi^*}\|X\|_{\mathcal H}^{p}-\mathbb{E}_{\pi^*}f_M(X)\underset{M\to\infty}{\longrightarrow} 0,
\]
uniformly in $t$ for the first term. Passing $t\to\infty$ and then $M\to\infty$ finally yields
\[
\mathbb{E}_{\pi_t}\|X\|_{\mathcal H}^{p}\longrightarrow \mathbb{E}_{\pi^*}\|X\|_{\mathcal H}^{p}.
\]
\end{proof}

\subsection*{Proof of Theorem \ref{mainthm:posterior_contraction}}

\begin{proof}

Assumptions \textbf{(A)} and \textbf{(B)} guarantee that the posterior distribution is the stationary distribution of (\ref{Langevin_eq}). We now control the moments of this SPDE:
For any given $\alpha>0$, Itô's formula for Hilbert space-valued processes (see Theorem \ref{theorem:ito-formula}) applied to (\ref{Langevin_eq}) yields
\begin{equation}
    \begin{aligned}\label{ito_formula}
        e^{\alpha t}\|\theta_t-\theta^*\|_\mathcal{H}^2 
        &= \tfrac{2}{\sqrt{n}}\int_0^te^{\alpha s}\langle\theta_s-\theta^*, \, dW^{Q}_s\rangle_\mathcal{H} && \text{(J1)} \\
        & + \int_0^t \alpha e^{\alpha s}\|\theta_s-\theta^*\|_\mathcal{H}^2 \, ds && \text{(J2)} \\
        & + \int_0^t e^{\alpha s}\langle \theta_s-\theta^*, Q\nabla_\mathcal{H} F_n(\theta_s) \rangle_\mathcal{H} \, ds && \text{(J3)} \\
        & - \tfrac{1}{n} \int_0^t e^{\alpha s}\langle \theta_s-\theta^*, \theta_s + Q\nabla_\mathcal{H} V(\theta_s) \rangle_\mathcal{H} \, ds && \text{(J4)} \\
        & + \tfrac{\operatorname{tr}(Q)}{n} \int_0^t e^{\alpha s} \, ds. && \text{(J5)} 
    \end{aligned}
\end{equation}
We need to bound the terms in equation (\ref{ito_formula}) separately using our assumptions. We begin with (J3):

\begin{align*}
\text{(J3)} = &-\int_0^t e^{\alpha s}\langle \theta^*-\theta_s,Q\nabla_\mathcal{H} F_n(\theta_s) \rangle_\mathcal{H} \, ds \\
&\leq -\int_0^t e^{\alpha s}\langle \theta^*-\theta_s, Q\nabla_\mathcal{H} F(\theta_s) \rangle_\mathcal{H} \, ds
\\ & \quad + \int_0^t e^{\alpha s}\|\theta_s-\theta^*\|_\mathcal{H}\|Q(\nabla_\mathcal{H} F(\theta_s) - \nabla_\mathcal{H} F_n(\theta_s))\|_\mathcal{H} \, ds,
\end{align*}
where the inequality follows by Cauchy-Schwarz's inequality. Now, using the decomposition $F_n=F_n+F-F$, applying assumptions \textbf{(C.1)} and \textbf{(C.2)} we can upper-bound (J3) in a set $\Omega\in \mathcal{B}(\mathcal{H})$ of $\mathbb{P}^{n}_{\theta^{*}}$-probability of, at least, $1-\delta$ by 
\begin{align*}
&- \int_{0}^{t} \mu\left\|\theta_{s}-\theta^{*}\right\|_{\mathcal{H}}^{2} e^{\alpha s}\, ds + \int_{0}^{t}\left\|\theta_{s}-\theta^{*}\right\|_{\mathcal{H}}\left(\varepsilon_{1}(n, \delta)\left\|\theta_{s}-\theta^{*}\right\|_{\mathcal{H}}+\varepsilon_{2}(n, \delta)\right) e^{\alpha s}\,ds
\end{align*}
Now, using Young's inequality, \(xy\leq \tfrac{1}{2\tau}x^2 + \tfrac{\tau}{2}y^2\), with $x=\varepsilon_2(n,\delta)$, $y=\|\theta_s-\theta^*\|_\mathcal{H}$, and $\tau = \mu/3$, we arrive to the upper-bound for (J3):
\begin{align}
&-\int_{0}^{t} \mu\left\|\theta_{s}-\theta^{*}\right\|_{\mathcal{H}}^{2} e^{\alpha s} d s+ \int_{0}^{t}\left\|\theta_{s}-\theta^{*}\right\|_{\mathcal{H}}^{2}\left(\varepsilon_{1}(n, \delta)+\mu / 6\right) e^{\alpha s} d s+\frac{3 \varepsilon_{2}^{2}(n, \delta)}{2\mu} \int_{0}^{t} e^{\alpha s} ds.
\end{align}

The prior term can be controlled using assumption \textbf{(C.3)} along with Young's inequality:
\begin{equation*}
    \begin{aligned}
        \text{(J4)} &= -\tfrac{1}{n} \int_0^t e^{\alpha s}\langle \theta_s-\theta^*,\theta_s + Q\nabla_\mathcal{H} V \rangle_\mathcal{H} \, ds\\
        &\leq \tfrac{1}{n} \int_{0}^{t} e^{\alpha s} B \| \theta_{s} -\theta^{*} \|_{\mathcal{H}}\,  ds \\
        &\leq \int_{0}^{t} \frac{\mu}{6}\left\|\theta_{s}-\theta^{*}\right\|_{\mathcal{H}}^{2} e^{\alpha s} d s+\frac{3 B^{2}}{2n^{2} \mu} \int_{0}^{t} e^{\alpha s} d s.
    \end{aligned}
\end{equation*}
Combining the bounds for (J3) and (J4) with equation (\ref{ito_formula}), we have
\begin{align*}
e^{\alpha t}\|\theta_t-\theta^*\|_\mathcal{H}^2 \leq &\big( \alpha - 2\mu/3 + \varepsilon_1(n,\delta)\big)\int_0^t\|\theta_s-\theta^*\|_\mathcal{H}^2 e^{\alpha s} ds \\ &+ \big( \tfrac{\operatorname{tr}(Q)}{n} + \tfrac{3B^2}{2\mu n^2} + \tfrac{3\varepsilon^2_2(n,\delta)}{2\mu}\big)\frac{e^{\alpha t} -1}{\alpha} + \text{(J1)}.
\end{align*}

Since we assumed $\varepsilon_1(n,\delta)\leq \mu/6$, taking $\alpha \leq \mu/2$ is enough to get rid of the first term. This yields

\begin{equation}\label{U_n + martingale}
        e^{\alpha t}\|\theta_t-\theta^*\|_\mathcal{H}^2 \leq \underbrace{\Big( \tfrac{\operatorname{tr}(Q)}{n} + \tfrac{3B^2}{2\mu n^2} + \tfrac{3\varepsilon^2_2(n,\delta)}{2\mu}\Big)}_{U_n}\frac{e^{\alpha t} -1}{\alpha} + \tfrac{2}{\sqrt{n}}\underbrace{\int_0^te^{\alpha s}\langle\theta_s-\theta^*, \, dW^{Q}_s\rangle_\mathcal{H}}_{M_t} .
\end{equation}

We now need to bound the martingale term, $M_{t}$. First we note that $\int_{0}^{t}e^{\alpha s}\langle\theta_s-\theta^*, \, dW^{Q}_s\rangle_\mathcal{H}$ is a real martingale (see Theorem \ref{theorem:real-valued-martingale}). This can be shown using the Karhunen-Loève expansion of the $Q$-Wiener process \citep{liu2015stochastic}. Namely,
\begin{equation}\label{karhunen-loeve}
    W^{Q}_t = \sum_{k\geq 1} \sqrt{\lambda_k} e_k \beta_t^k,
\end{equation}
where $\{e_k\}_{k\geq 1}$ is an orthonormal basis of $\mathcal{H}$ consisting of eigenvectors of $Q$ with corresponding eigenvalues $\{\lambda_k\}_{k\geq 1}$, and $\beta^k$ are independent real-valued Brownian motions.
Using expansion (\ref{karhunen-loeve}) and a monotone convergence argument, we have

\begin{equation*}
    M_t=\int_0^te^{\alpha s}\sum_{k\geq 1}\sqrt{\lambda_k}\langle\theta_s-\theta^*, \, e_k\rangle_\mathcal{H} \,d\beta_t^k =     \sum_{k\geq 1}\int_0^te^{\alpha s}\sqrt{\lambda_k}\langle\theta_s-\theta^*, \, e_k\rangle_\mathcal{H} \,d\beta_t^k.
\end{equation*}
Since each $\int_0^te^{\alpha s}\sqrt{\lambda_k}\langle\theta_s-\theta^*, \, e_k\rangle_H \,d\beta_t^k$ is a real-valued martingale, the (convergent) sum of independent martingales is also a martingale.
Hence we can apply the Burkholder-Davis-Gundy inequality for real martingales \citep{le2016brownian} to bound its moments. For any $p\geq 4$,

\begin{equation*}
    \begin{aligned}
        \mathbb{E}\left[\sup_{t\leq T} |M_t|^{\tfrac{p}{2}}\right]&\leq (pC)^{\tfrac{p}{4}}\mathbb{E}\left(\left[\int_0^T e^{2\alpha s}\langle Q(\theta_s-\theta^*), \theta_s-\theta^*\rangle_\mathcal{H} ds\right]^{\tfrac{p}{4}}\right)\\
        &\leq (pC)^{\tfrac{p}{4}}\mathbb{E}\left(\left[\int_0^T e^{2\alpha s}\| Q(\theta_s-\theta^*)\|_\mathcal{H}\| \theta_s-\theta^*\|_\mathcal{H} ds\right]^{\tfrac{p}{4}}\right)\\
        &\leq (pC\|Q\|_{\text{op}})^{\tfrac{p}{4}}\mathbb{E}\left(\left[\int_0^T e^{2\alpha s} \|\theta_s-\theta^*\|^2_\mathcal{H} ds\right]^{\tfrac{p}{4}}\right)\\
        &\leq \left(\frac{pC\|Q\|_{\text{op}}e^{\alpha T}}{\alpha}\right)^{\tfrac{p}{4}}\mathbb{E}\left(\left[\sup_{t\leq T} e^{\alpha t}\|\theta_t-\theta^*\|_\mathcal{H}^2 \right]^{\tfrac{p}{4}}\right).
    \end{aligned}
\end{equation*}

Finally, we arrive to the following bound:

\begin{equation}
\begin{aligned}
&\mathbb{E}\left[\left(\sup_{t\leq T} e^{\alpha t/2} \|\theta_t-\theta^*\|_\mathcal{H} \right)^p \right] 
\\&\leq \mathbb{E}\left[\left(\sup_{t\leq T} e^{\alpha t} \|\theta_t-\theta^*\|^2_\mathcal{H} \right)^{\tfrac{p}{2}} \right]\\
&\leq 2^{\tfrac{p}{2}}\left[\left(U_n\frac{e^{\alpha T}}{\alpha} \right)^{\tfrac{p}{2}} + \mathbb{E}\left(\tfrac{2}{\sqrt{n}}\sup_{t\leq T}M_t\right)^{\tfrac{p}{2}}\right]\\
&\leq 2^{\tfrac{p}{2}}\left[\left(U_n\frac{e^{\alpha T}}{\alpha} \right)^{\tfrac{p}{2}} + \left(\frac{pC\|Q\|_{\text{op}}e^{\alpha T}}{\alpha n}\right)^{\tfrac{p}{4}}\mathbb{E}\left(\left[\sup_{t\leq T} e^{\alpha t}\|\theta_t-\theta^*\|_\mathcal{H}^2 \right]^{\tfrac{p}{4}}\right)\right]\\
&\leq 2^{\tfrac{p}{2}}\left[\left(U_n\frac{e^{\alpha T}}{\alpha} \right)^{\tfrac{p}{2}} + \frac{1}{2}\left(\frac{2pC\|Q\|_{\text{op}}e^{\alpha T}}{\alpha n}\right)^{\tfrac{p}{2}}+\frac{1}{2^{\tfrac{p}{2}+ 1}}\mathbb{E}\left(\left[\sup_{t\leq T} e^{\alpha t}\|\theta_t-\theta^*\|_\mathcal{H}^2 \right]^{\tfrac{p}{2}}\right)\right].
\end{aligned}
\end{equation}

Thus 
\begin{equation}\label{boundforfurthercontraction}
\begin{aligned}
    \left(\mathbb{E}\left[ \|\theta_T-\theta^*\|^p_\mathcal{H} \right]\right)^{\tfrac{1}{p}} &\leq  
    e^{-\alpha T/2} \mathbb{E}\left[\left(\sup_{t\leq T} e^{\alpha t/2} \|\theta_t-\theta^*\|_\mathcal{H} \right)^p \right]^{\tfrac{1}{p}}\\ 
    &\leq e^{-\alpha T/2}\left[2\left(U_n\frac{2e^{\alpha T}}{\alpha} \right)^{\tfrac{p}{2}} + \left(\frac{4pC\|Q\|_{\text{op}}e^{\alpha T}}{\alpha n}\right)^{\tfrac{p}{2}}\right]^{\tfrac{1}{p}}\\
    &\leq K\left(\sqrt{\frac{U_n}{\mu}} +\sqrt{\frac{p\|Q\|_{\text{op}}}{n\mu}} \right)=:B_n
\end{aligned}
\end{equation}
for a universal constant $K>0$. This implies that $\sup_{t\geq 0} \mathbb{E}\left[ \|\theta_t-\theta^*\|^{p}_\mathcal{H} \right]<\infty$ for any $p\geq 2$, and Proposition \ref{auxiliary_moment_prop} applies. Thus, using Markov's inequality,
\begin{equation*}
  \Pi\biggl(\|\theta-\theta^{\!*}\|_\mathcal{H}\geq r\;\Bigm| \; X_1^n\biggr)\;\le\;\frac{\mathbb{E}_{\Pi}\left[ \|\theta-\theta^*\|^{p}_\mathcal{H} \;|\; X_1^n\right]}{r^p}\leq\frac{B_n^p}{r^p},
\end{equation*}
for any $r>0$. If we impose $\frac{B_n^p}{r^p}\leq \delta$, this implies $r\geq B_n\delta^{-1/p}$. Now, taking $p=\log\left(\tfrac{1}{\delta}\right)$ we have $r\geq B_n e$, absorbing $e$ in the constant, and remembering that the bound is valid in a set $\Omega\in \mathcal{B}(\mathcal{H})$ of $\mathbb{P}^{n}_{\theta^{*}}$-probability of, at least, $1-\delta$, we end up with
\begin{equation*}
  \mathbb{P}^{n}_{\theta^{*}}\left(\Pi\left(\|\theta-\theta^{\!*}\|_\mathcal{H}\geq K'\left(\sqrt{\frac{U_n}{\mu}} +\sqrt{\frac{\log\left(\tfrac{1}{\delta}\right)\|Q\|_{\text{op}}}{n\mu}} \right)\;\Bigm| \; X_1^n\right)\;\leq\delta\right) \geq 1-\delta,
\end{equation*}
from which the result follows.

\end{proof}

\subsection*{Proof of Theorem \ref{weak posterior contraction}}

\begin{proof}
For any $p \geq 2$, we define the functions on the positive real line $(0, \infty)$

$$
\nu_{(p)}(r):=\psi\left(r^{\frac{1}{p-1}}\right) r^{\frac{p-2}{p-1}}, \quad \text { and } \quad \tau_{(p)}\left(r^{p-1} \zeta(r)\right):=r^{p-2} \psi(r) .
$$

By Assumption (W.2), the function $r \mapsto r^{p-1} \zeta(r)$ is strictly increasing and surjective function that maps from $[0,+\infty)$ to $[0,+\infty)$. Therefore, it is invertible and the function $\tau_{(p)}^{-1}$ is well-defined.

Now we claim that for any $p \geq 2$, the functions $\nu_{(p)}$ and $\tau_{(p)}$ are convex and strictly increasing, and that furthermore, the expectation $\mathbb{E}\left[\left\|\theta_{t}-\theta^{*}\right\|_\mathcal{H}^{p}\right]$ is upper bounded by the integral
For any $p\geq 2$, we apply Itô's formula to the functional $\|\theta_t-\theta^*\|_\mathcal{H}^p$, resulting in 

\begin{equation}
    \begin{aligned}\label{weak_ito_formula}
        \|\theta_t-\theta^*\|_\mathcal{H}^p 
        &\leq  \tfrac{p}{\sqrt{n}}\int_0^t\|\theta_s-\theta^*\|^{p-2}_\mathcal{H}\langle\theta_s-\theta^*, \, dW^{Q}_s\rangle_\mathcal{H} && \text{(J1)} \\
        &-\tfrac{p}{2}\int_0^t \|\theta_s-\theta^*\|_\mathcal{H}^{p-2}\langle \theta^*-\theta_s, Q\nabla_\mathcal{H} F(\theta_s)\rangle_\mathcal{H} \, ds && \text{(J2)}\\
        & + \tfrac{p}{2}\int_0^t \|\theta_s-\theta^*\|_\mathcal{H}^{p-2}\langle \theta^*-\theta_s, Q\nabla_\mathcal{H} F(\theta_s)-Q\nabla_\mathcal{H} F_n(\theta_s) \rangle_\mathcal{H} \, ds && \text{(J3)} \\
        & + \tfrac{p}{2n} \int_0^t \|\theta_s-\theta^*\|_\mathcal{H}^{p-2}\langle \theta_s-\theta^*, \theta_s + Q\nabla_\mathcal{H} V(\theta_s) \rangle_\mathcal{H} \, ds && \text{(J4)} \\
        & + \tfrac{p\left(\operatorname{tr}(Q) + (p-2)\|Q\|_{\operatorname{op}}\right)}{2n} \int_0^t \|\theta_s-\theta^*\|_\mathcal{H}^{p-2} \, ds. && \text{(J5)} 
    \end{aligned}
\end{equation}
We now bound the expectation of each term separately in terms of 
$$\mathcal{R}_p(s):=\mathbb{E}\left[\|\theta_s-\theta^*\|_\mathcal{H}^{p-2}\psi(\|\theta_s-\theta^*\|_\mathcal{H})\right].$$ 
From the weak convexity assumption along with Fubini's theorem, for (J2) we have

\begin{equation}
    -\frac{p}{2}\mathbb{E}\left[\int_0^t \|\theta_s-\theta^*\|_\mathcal{H}^{p-2}\langle\theta^*-\theta_s,Q\nabla_\mathcal{H}F(\theta_s)\rangle_\mathcal{H}\,ds\right]\leq -\frac{p}{2}\int_0^t\mathcal{R}_p(s)\,ds.
\end{equation}

where $R_{p}(s):=\mathbb{E}\left[\left\|\theta_{s}-\theta^{*}\right\|_\mathcal{H}^{p-2} \psi\left(\left\|\theta_{s}-\theta^{*}\right\|_\mathcal{H}\right)\right]$.\\
For (J3) we have
\begin{align*}
\frac{p}{2}\mathbb{E}\int_0^t \|\theta_s-\theta^*\|_\mathcal{H}^{p-2}&\langle\theta^*-\theta_s,Q\nabla_\mathcal{H}F(\theta_s)-Q\nabla_\mathcal{H}F_n(\theta_s)\rangle_\mathcal{H}ds \\
&\leq \frac{p}{2}\varepsilon(n,\delta)\int_0^t\mathbb{E}\left[\|\theta_s-\theta^*\|_\mathcal{H}^{p-1}\zeta\left(\|\theta_s-\theta^*\|_\mathcal{H}\right)\right]\,ds\\
&\leq \frac{p}{2}\varepsilon(n,\delta)\int_0^t\tau_p^{-1} \mathbb{E}\left[\tau_p\left(\|\theta_s-\theta^*\|_\mathcal{H}^{p-1}\zeta\left(\|\theta_s-\theta^*\|_\mathcal{H}\right)\right)\right]\,ds\\
&=\frac{p}{2}\varepsilon(n,\delta)\int_0^t\tau_p^{-1}\left(\mathcal{R}_p(s)\right)\,ds,
\end{align*}

where we used Assumption \textbf{(W.2)} and the convexity of $\tau_p$ along with Jensen's inequality.
For (J4), we use Assumption (B) and the fact that $\nu_p$ is strictly increasing to obtain the bound
\begin{equation}
\begin{aligned}
    \frac{p}{2n}\mathbb{E}\left[ \int_0^t \|\theta_s-\theta^*\|_\mathcal{H}^{p-2}\langle \theta_s-\theta^*, \theta_s + Q\nabla_\mathcal{H} V(\theta_s) \rangle_\mathcal{H} \, ds\right]&\leq \frac{pB}{2n}\int_0^t\mathbb{E}\left[\|\theta_s-\theta^*\|_\mathcal{H}^{p-1}\right]\,ds\\
    &\leq \frac{pB}{2n}\int_0^t\nu_p^{-1}\left(\mathcal{R}_p(s)\right)\,ds.
\end{aligned}
\end{equation}

The term (J1) is a real martingale, hence
\begin{equation}
    \frac{p}{\sqrt{n}}\mathbb{E}\left[\int_0^t\|\theta_s-\theta^*\|^{p-2}_\mathcal{H}\langle\theta_s-\theta^*, \, dW^{Q}_s\rangle_\mathcal{H}\right]=0.
\end{equation}

Finally, we can bound (J5) as

\begin{equation}
\begin{aligned}
    \tfrac{p\left(\operatorname{tr}(Q) + (p-2)\|Q\|_{\operatorname{op}}\right)}{2n} \mathbb{E}\left[\int_0^t \|\theta_s-\theta^*\|_\mathcal{H}^{p-2} \, ds\right]&\leq \tfrac{p\left(\operatorname{tr}(Q) + (p-2)\|Q\|_{\operatorname{op}}\right)}{2n}\int_0^t\nu_p^{-1}\left(\mathcal{R}_p(s)\right)^{\frac{p-2}{p-1}}\,ds,
\end{aligned}
\end{equation}

where we used the following moment bound based on Hölder's inequality and the fact that $\nu_p$ is convex and strictly increasing:
\begin{equation}
    \mathbb{E}\left[\|\theta_s-\theta^*\|_\mathcal{H}^{p-2}\right]\leq \left(\mathbb{E}\left[\|\theta_s-\theta^*\|_\mathcal{H}^{p-1}\right]\right)^{\frac{p-2}{p-1}}\leq \nu_p^{-1}\left(\mathcal{R}_p(s)\right)^{\frac{p-2}{p-1}}.
\end{equation}

Putting everything together, we obtain the final moment bound:

\begin{align}\label{eq54}
\mathbb{E}\left[\|\theta_s-\theta^*\|_\mathcal{H}^p\right] 
\leq&-\frac{p}{2} \int_{0}^{t}R_{p}(s)ds+\frac{p}{2} \int_{0}^{t}\varepsilon(n, \delta) \tau_{p}^{-1}\left(R_{p}(s)\right)ds\\&+\frac{p}{2} \int_{0}^{t}\frac{B}{n} \nu_{p}^{-1}\left(R_{p}(s)\right)ds
+\frac{p}{2} \int_{0}^{t}\tfrac{\operatorname{tr}(Q) + (p-2)\|Q\|_{\operatorname{op}}}{n} \nu_{p}^{-1}\left(R_{p}(s)\right)^{\frac{p-2}{p-1}} ds.
\end{align}

Once we've established bound (\ref{eq54}), the rest of the proof follows exactly the steps of the proof of Theorem 2 in \cite{mou2024diffusion}, because the remaining arguments only depend on the properties of the scalar functions $\nu$ and $\tau$, and are independent of the parameter space.

Hence we obtain that 
\begin{equation}
    \underset{t\rightarrow\infty}{\lim}\left(\mathbb{E}\left[\|\theta_t-\theta^*\|_\mathcal{H}^p\right]\right)^{\frac{1}{p}}\leq z^*_p(n,\delta),
\end{equation}

where $z^*_p(n,\delta)$ is the only positive solution to the equation
\begin{equation}
    \psi(z)=\varepsilon(n,\delta)\zeta(z)z+\frac{B}{n}z + \frac{\operatorname{tr}(Q) + p\|Q\|_{\operatorname{op}}}{n}.
\end{equation}

Using Markov's inequality and taking $p=\log(1/\delta)$ as in the proof of the previous theorem concludes the proof.
\end{proof}

\subsection*{Proof of Theorem \ref{bernstein von mises}}

\begin{proof} 
First, note that each eigenvalue of $(Q^{-1}+nH^{*})^{-1}$ verifies 
\begin{align}\label{spectral_norm}
\frac{\mu_{m}}{1+n\lambda_{m}\mu_{m}} 
\leq \frac{\|Q\|_{\text{op}}}{1+n\lambda_{m}\mu_{m}}\leq \frac{\|Q\|_{\text{op}}}{n\lambda_{m}\mu_{m}} \leq \frac{\|Q\|_{\text{op}}}{n\mu}, 
\end{align}
so the spectral norm of $(Q^{-1}+nH^{*})^{-1}$ is upper bounded $\frac{\|Q\|_{\text{op}}}{n\mu}$. Now, since $\Pi(\cdot|X^{n}) <<\gamma_{\widehat\theta^{(n)}}$, we can combine (\ref{spectral_norm}) Gross’s Logarithmic Sobolev Inequality \citep{gross1975logarithmic} and conclude that
\begin{equation}\label{LSI}
\text{KL}(\Pi(\cdot|X^{n})||\gamma_{\widehat\theta^{(n)}})
  \leq
  \frac{\|Q\|_{\text{op}}}{2n\mu}\,  \mathbb{E}_{\Pi(\cdot|X^{n})}\Bigg[\Big\|\nabla_{\mathcal{H}_Q}\log\left(\frac{d\Pi(\cdot|X^{n})}{d\gamma_{\widehat\theta^{(n)}}}\right) 
  \Big\|_{\mathcal{H}_{Q}}^{2}\Bigg].
\end{equation}
Computing the Right-Hand-Side of the above inequality and bounding it by negative powers of $n$ will give the result. Now, by Feldman-Hajék Theorem, the Radon-Nikodym derivative of $\mu$ with respect to $\gamma_{\widehat{\theta}^{(n)}}$ is 
$$
\frac{d\mu}{d\nu}(\theta) =
\zeta_{n}(Q,H^{*},n) 
\cdot e^{
- \frac{1}{2} \| \widehat{\theta}^{(n)}\|_{\mathcal{H}_{Q}}^{2}
+ \frac{1}{2}\langle \theta - \widehat{\theta}^{(n)},nH^{*}(\theta - \widehat{\theta}^{(n)})\rangle_{\mathcal{H}}
-\langle \theta-\widehat{\theta}^{(n)},Q^{-1}\widehat{\theta}^{(n)}\rangle_{\mathcal{H}}},
$$
where $\zeta_{n}(Q,H^{*},n)$ is a normalizing constant and, by assumption, the Radon-Nikodym derivative of the posterior with respect to $\mu$ is 
$$\frac{d\Pi}{d\mu}(\theta|X^{n}) = \frac{e^{-nF_{n}(\theta)-V(\theta)}}{Z}.$$
This way, putting $\|\cdot\|_{nH^{*}} = \|(nH^{*})^{1/2}(\cdot)\|_H$, we will have that the
Radon-Nikodym derivative of the posterior with respect to the limiting Gaussian, $\frac{d\Pi}{d\gamma_{\widehat{\theta}^{(n)}}}(\theta|X^{n})$, will be proportional to 
\begin{align}
\exp{\left\lbrace
- \frac{1}{2} \| \widehat{\theta}^{(n)}\|_{\mathcal{H}_{Q}}^{2}
+\frac{1}{2}\|\theta-\widehat{\theta}^{(n)}\|^{2}_{nH^{*}}
-\langle  \theta-\widehat{\theta}^{(n)},Q^{-1}\widehat{\theta}^{(n)}\rangle_{H}
-\Big(nF_{n}(\theta)+V(\theta)\Big)
\right\rbrace}.
\end{align}
Now, taking the (Gross-Malliavin) derivative $\nabla_{Q}$ (see \cite{malliavin1}, \cite{malliavin2} and \cite{NualartMalliavin}) about the equivalence between the Malliavin-Gross derivative and the image of the Fréchet derivative on the Cameron-Martin Space by $Q^{1/2}$ of the log-Radon-Nikodym derivative, we will have
\begin{align}\label{lsi0}
\nabla_{\mathcal{H}_Q}\,
\log\!\Biggl(
  \frac{d\Pi(\,\cdot\mid X^{n}\,)}{d\gamma_{\widehat\theta^{(n)}}}
\Biggr)
= -\widehat\theta^{(n)}
   + nQH^{*}\bigl[\theta - \widehat\theta^{(n)}\bigr] -nQ\nabla_{\mathcal{H}}F_{n}(\theta)
         -Q\nabla_{\mathcal{H}}V(\theta).
\end{align}
Since $\widehat{\theta}^{(n)}$ is a minimizer of the Onsager-Machlup functional, we will have $D_\mathcal{H_{Q}}\mathcal{J}_{n}(\widehat{\theta}^{(n)})=0$, so $-\widehat{\theta}^{(n)} = nQ\nabla_{\mathcal{H}}F_{n}(\widehat{\theta}^{(n)}) + Q\nabla_{\mathcal{H}}V(\widehat{\theta}^{(n)})$. Substituting in the above expression:
\begin{align*}
\nabla_{\mathcal{H}_Q}\log\left(\frac{d\Pi(\theta|X^{n})}{d\gamma_{\widehat\theta^{(n)}}}\right)  = &nQ\left[H^{*}[\theta - \widehat{\theta}^{(n)}]+\nabla_\mathcal{H}F_{n}(\widehat{\theta}^{(n)})-\nabla_\mathcal{H}F_{n}(\theta)\right]\\&+ \left[Q\nabla_\mathcal{H}V(\widehat{\theta}^{(n)})-\nabla_{\mathcal{H}}V(\theta)\right].
\end{align*}
Then, 
taking the $\mathcal{H_{Q}}$ norms, taking into account that $\|\cdot\|_{\mathcal{H_{Q}}}=\|Q^{-1/2}[\cdot]\|_{\mathcal{H}}$,
applying triangle and Young's inequalities and the definition of operator norm, we will have 
\begin{align}\label{bonaineq}
\left\|\nabla_{\mathcal{H}_Q}\log\!\left(\frac{d\Pi(\theta\mid X^{n})}{d\gamma_{\widehat\theta^{(n)}}}\right)\right\|_{\mathcal{H}_{Q}}^{2}
\leq& \,\, 2
\Bigl(
  n^{2}\left\|Q^{1/2}\left(H^{*}(\theta-\widehat{\theta}^{(n)})+\nabla_{\mathcal{H}}F_{n}(\widehat{\theta}^{(n)}) -\nabla_{\mathcal{H}}F_{n}(\theta)\right)\right\|_{\mathcal{H}}^{2}\\ &+\left\|Q^{1/2}\left(\nabla_{\mathcal{H}}V(\widehat{\theta}^{(n)})-\nabla_{\mathcal{H}}V(\theta)\right)\right\|_{\mathcal{H}}^{2}
\Bigr). \nonumber
\end{align}
We now upper bound each term in the Right-Hand-Side of the inequality. First, note that, since  $\nabla_{\mathcal{H}}V$ is Lipschitz, we will have 
\begin{align*}
\left\|Q^{1/2}\left(\nabla_{\mathcal{H}}V(\widehat{\theta}^{(n)})-\nabla_{\mathcal{H}}V(\theta)\right)\right\|_{\mathcal{H}}^{2} \leq L_{2}^{2}\left\|\theta - \widehat{\theta}^{(n)}\right\|_\mathcal{H}^{2} \leq 2L_{2}^{2}\Big(\left\|\theta - \theta^{*}\right\|_\mathcal{H}^{2} + \left\|\theta^{*} - \widehat{\theta}^{(n)}\right\|_H^{2}\Big),
\end{align*}

where the second inequality follows from Young's inequality. Second, by Riesz representation theorem, the Hilbert space norm of the other term in the Right-Hand-Side in (\ref{bonaineq}) equals the operator norm for the associated operator, that is 
\begin{align}
&\left\|Q^{1/2}\left(H^{*}(\theta-\widehat{\theta}^{(n)}) +\nabla_{\mathcal{H}}F_{n}(\widehat{\theta}^{(n)})-\nabla_{\mathcal{H}}F_{n}(\theta)\right)\right\|_{\mathcal{H}} 
\\ &= \left\|Q^{1/2}\left(H^{*}(\theta-\widehat{\theta}^{(n)}) +\nabla_{\mathcal{H}}F_{n}(\widehat{\theta}^{(n)})-\nabla_{\mathcal{H}}F_{n}(\theta)\right)\right\|_{\text{op}}.
\end{align}
Now,  using the Fundamental Theorem of Calculus for Hilbert Spaces, we will have
\begin{align}\label{calculushilbert}
&-[\nabla_\mathcal{H}F_{n}(\theta)-\nabla_\mathcal{H}F_{n}(\widehat{\theta}^{(n)})]= -\int_{0}^{1}\left( \nabla_\mathcal{H }^{2} F_{n}\left(\gamma \theta+(1-\gamma) \widehat{\theta}^{(n)}\right)\right)\left[\theta-\widehat{\theta}^{(n)}\right] d \gamma 
.
\end{align}

Then, applying Minkowski's integral inequality and using the definition of operator norm, we will have
\begin{align*}
&\left\|Q^{1/2}\left(H^{*}(\theta-\widehat{\theta}^{(n)}) +\nabla_{\mathcal{H}}F_{n}(\widehat{\theta}^{(n)})-\nabla_{\mathcal{H}}F_{n}(\theta)\right)\right\|_{\mathrm{op}}^{2} \\
&= \Big\|\int_{0}^{1}Q^{1/2}\Big(H^{*}-\nabla_{\mathcal{H}}^{2} F_{n}\left(\gamma \theta+(1-\gamma) \widehat{\theta}^{(n)}\right)\Big)\left[\theta-\widehat{\theta}^{(n)}\right]\,d\gamma\Big\|_{\mathrm{op}}^{2} \\
&\leq 
\Big(\int_{0}^{1}\Big\|Q^{1/2}\Big(H^{*}- \nabla_{\mathcal{H}}^{2} F_{n}\left(\gamma \theta+(1-\gamma) \widehat{\theta}^{(n)}\right)\Big)\left[\theta-\widehat{\theta}^{(n)}\right]\Big\|_{\mathrm{op}}\,d\gamma\Big)^{2} \\
&\leq 
\Big(\int_{0}^{1}\Big\|Q^{1/2}\left(H^{*}- \nabla_{\mathcal H}^{2} F_{n}\left(\gamma \theta+(1-\gamma) \widehat{\theta}^{(n)}\right)\right)\Big\|_{\mathrm{op}}
\left\|\theta-\widehat{\theta}^{(n)}\right\|_{\mathcal{H}}\,d\gamma\Big)^{2} \\
&\leq 
2\Big(\left\|\theta^{*}-\widehat{\theta}^{(n)}\right\|_{\mathcal{H}}^{2}+ \Big\|\theta-\theta^{*} \Big\|_{\mathcal{H}}^{2}\Big)
\Big(\int_{0}^{1}\Big\|Q^{1/2}\left(H^{*}- \nabla_{\mathcal H}^{2} F_{n}\left(\gamma \theta+(1-\gamma) \widehat{\theta}^{(n)}\right)\right)\Big\|_{\mathrm{op}}\,d\gamma\Big)^{2}.
\end{align*}

Adding and subtracting $Q^{1/2}\nabla_{\mathcal{H}}^{2} F\left(\gamma \theta+(1-\gamma) \widehat{\theta}^{(n)}\right)$ inside the operator norm and  applying triangle inequality, we can upper-bound $\left\|\nabla_\mathcal{H}^{2}  F_{n}\left(\gamma \theta+(1-\gamma) \widehat{\theta}^{(n)}\right)-H^{*}\right\|_{\text{op}}$ by
\begin{align*}
&\underbrace{\left\|\nabla_H^{2} F\left(\gamma \theta+(1-\gamma) \widehat{\theta}^{(n)}\right)-H^{*}\right\|_{\text{op}}}_{B_{1}}+\underbrace{\left\|\nabla^{2}_H F_{n}\left(\gamma \theta+(1-\gamma) \widehat{\theta}^{(n)}\right)-\nabla^{2}_{\mathcal{H}}  F\left(\gamma \theta+(1-\gamma) \widehat{\theta}^{(n)}\right)\right\|_{\text{op}}}_{B_{2}}
\end{align*}
Now, by \textbf{(BvM.1)} and  \textbf{(BvM.2)}, we will have upper bounds on $B_{1}$ nd  $B_{2}$, respectively, so    the RHS of the above inequality will be upper bounded by

\begin{equation}\label{topbound}
n\Big(A+\varepsilon_{1}^{(2)}(n, \delta)\Big)\Big(\gamma\left\|\theta - \theta^{*}\right\|_\mathcal{H} + \left\|\theta^{*} - \widehat{\theta}^{(n)}\right\|_\mathcal{H}\Big) + n\varepsilon_{2}^{(2)}(n, \delta)
\end{equation}
in a set $\Omega_{1} \subset \mathcal{H}$ of  $\mathbb{P}_{\theta^{*}}$-probability of at least $1-\delta$. Substituting into the integral and applying Young's Inequality twice, we will have $\forall a >0$ 
\begin{align*}
&n^{2}\Big(\int_{0}^{1}\Big\| H^{*}- \nabla_\mathcal{H}^{2} F_{n}\left(\gamma \theta+(1-\gamma) \widehat{\theta}^{(n)}\right)\Big\|_{\text{op}}d\gamma\Big)^{2} \\
&\leq 4n^{2}\Big(A+\varepsilon_{1}^{(2)}(n, \delta)\Big)^{2}\Big(\int_{0}^{1}a\gamma^{2}d\gamma\left\|\theta -\theta^{*}\right\|_H^{2} + \frac{1}{a}\left\|\theta^{*} - \widehat{\theta}^{(n)}\right\|_\mathcal{H}^{2}\Big) + 2[n\varepsilon_{2}^{(2)}(n, \delta)]^{2}
\\ &= 4n^{2}\Big(A+\varepsilon_{1}^{(2)}(n, \delta)\Big)^{2}\Big(\frac{a}{3}\left\|\theta - \theta^{*}\right\|_\mathcal{H}^{2} + \frac{1}{a}\left\|\theta^{*} - \widehat{\theta}^{(n)}\right\|_{\mathcal{H}}^{2}\Big) + 2[n\varepsilon_{2}^{(2)}(n, \delta)]^{2} 
\end{align*}

Putting $a=\sqrt{3}$ and substituting the obtained bounds into (\ref{bonaineq}), we will have (in $\Omega_{1}$) 
\begin{align*}
&\Big\|\nabla_{\mathcal{H}_{Q}}\log\left(\frac{d\Pi(\theta|X^{n})}{d\gamma_{\widehat\theta^{(n)}}}\right) 
  \Big\|_{\mathcal{H}_{Q}}^{2} 
\\&\leq \mathcal{H}_{1}(n,Q)\Big(\left\|\theta^{*}-\widehat{\theta}^{(n)}\right\|_H^{2}+ \Big\|\theta-\theta^{*} \Big\|_H^{2}\Big)^{2} + \mathcal{H}_{2}(n,Q)\Big(\left\|\theta^{*}-\widehat{\theta}^{(n)}\right\|_H^{2}+ \Big\|\theta-\theta^{*} \Big\|_H^{2}\Big)
  \\ & \leq 2\mathcal{H}_{1}(n,Q)\Big(\left\|\theta^{*}-\widehat{\theta}^{(n)}\right\|_H^{4}+ \Big\|\theta-\theta^{*} \Big\|_H^{4}\Big) + \mathcal{H}_{2}(n,Q)\Big(\left\|\theta^{*}-\widehat{\theta}^{(n)}\right\|_H^{2}+ \Big\|\theta-\theta^{*} \Big\|_H^{2}\Big),
\end{align*}
where 
\begin{align*}
& \mathcal{H}_{1}(n,Q):= \dfrac{16\sqrt{3}}{3}n^{2}\Big(A+\varepsilon_{1}^{(2)}(n, \delta)\Big)^{2},
\  \mathcal{H}_{2}(n,Q):=2\left(\|Q\|_{\text{op}}L_{2}^{2} + 2[n\varepsilon_{2}^{(2)}(n, \delta)]^{2}\right),
\end{align*}
and the second inequality follows from Young's Inequality. Now, by assumptions (\textbf{E.1}) and (\textbf{E.2}) there exists a set $\Omega_{2}\subset \mathcal{H}$ with $\mathbb{P}_{\theta^{*}}$-probability at least $1-2\delta$ in which 
\begin{align*}
\mathbb{E}_{\Pi(\cdot|X^{n})}&\left[2\mathcal{H}_{1}(n,Q)\Big(\left\|\theta^{*}-\widehat{\theta}^{(n)}\right\|_H^{4}+ \Big\|\theta-\theta^{*} \Big\|_H^{4}\Big) + \mathcal{H}_{2}(n,Q)\Big(\left\|\theta^{*}-\widehat{\theta}^{(n)}\right\|_H^{2}+ \Big\|\theta-\theta^{*} \Big\|_H^{2}\Big)\right]
\\ 
& \leq 2\mathcal{H}_{1}(n,Q)n^{-4\alpha}+ \mathcal{H}_{2}(n,Q)n^{-2\alpha}
\\&= c_{1}\left[A^{2}+ \varepsilon_{1}^{(2)}(n,\delta)^{2}\right]n^{2-4\alpha}+ c_{2}\left[\varepsilon_{2}^{(2)}(n,\delta)^{2}+\dfrac{L_{2}^{2}}{n^{2}}\right]n^{2-2\alpha}.
\end{align*}
Indeed, it is enough to note that, for every pair of sets $A,B\in \mathcal{B}(\mathcal{H})$, we have $\mathbb{P}^{n}_{\theta^{*}}(A \cup B) = \mathbb{P}_{\theta^{*}}(A) + \mathbb{P}_{\theta^{*}}(B) - \mathbb{P}_{\theta^{*}}(A\cap B) \leq 1$, so if  $\mathbb{P}^{n}_{\theta^{*}}(A) \geq 1-\delta$ and $\mathbb{P}^{n}_{\theta^{*}}(B)  \geq 1-\delta$, then $\mathbb{P}_{\theta^{*}}(A\cap B) \geq \mathbb{P}_{\theta^{*}}(A) + \mathbb{P}_{\theta^{*}}(B) - 1 \geq 1 - 2\delta$. 
Substituting in (\ref{LSI}), we will have 
\begin{align}\label{LaplaceRate}
\mathbb{E}_{\Pi(\cdot|X^{n})}\left[\Big\|\nabla_{\mathcal{H}_{Q}}\log\left(\frac{d\Pi(\theta|X^{n})}{d\gamma_{\widehat\theta^{(n)}}}\right) \Big\|_{\mathcal{H}_{Q}}^{2}\right]&\leq  
\dfrac{c_{1}\|Q\|_{\text{op}}}{\mu}\left[A^{2}+ \varepsilon_{1}^{(2)}(n,\delta)^{2}\right]n^{1-4\alpha}\\ &+ \dfrac{c_{2}\|Q\|_{\text{op}}}{\mu}\left[\varepsilon_{2}^{(2)}(n,\delta)^{2}+\dfrac{L_{2}^{2}}{n^{2}}\right]n^{1-2\alpha},
\end{align} 
in the set $\Omega_{1}\cap \Omega_{2}$, and we get the result.
\end{proof}

\subsection*{Proof of Theorem \ref{theorem:laplace-approx}}
\begin{proof}
Using again the Logarithmic Sobolev Inequality, the $KL$ between the posterior and $\widehat{\gamma}_{\widehat\theta^{(n)}}$can be upper bounded by 
\begin{equation}\label{EmpiricalLSI}
\text{KL}(\Pi(\cdot|X^{n})||\widehat{\gamma}_{\widehat\theta^{(n)}})
  \leq
  \frac{1}{2n\lambda_{\min}(\widehat{H}_{n})}\,  \mathbb{E}_{\Pi(\cdot|X^{n})}\Bigg[\Big\|\nabla_{\mathcal{H}_Q}\log\left(\frac{d\Pi(\cdot|X^{n})}{d\widehat{\gamma}_{\widehat\theta^{(n)}}}\right) 
  \Big\|_{\mathcal{H}_{Q}}^{2}\Bigg].
\end{equation}
Now, since $\frac{d\Pi(\cdot|X^{n})}{d\widehat{\gamma}_{\widehat{\theta}^{(n)}}}=\frac{d\Pi(\cdot|X^{n})}{d\gamma_{\widehat{\theta}^{(n)}}}\cdot \frac{d\gamma_{\widehat{\theta}^{(n)}}}{d\widehat{\gamma}_{\widehat{\theta}^{(n)}}}$, applying Young's Inequality we will have 
\begin{align*}
\mathbb{E}_{\Pi(\cdot|X^{n})}\Bigg[\Big\|\nabla_{\mathcal{H}_Q}\log\left(\frac{d\Pi(\cdot|X^{n})}{d\widehat{\gamma}_{\widehat\theta^{(n)}}}\right) 
  \Big\|_{\mathcal{H}_{Q}}^{2}\Bigg] \leq  \ & 2\mathbb{E}_{\Pi(\cdot|X^{n})}\Bigg[\Big\|\nabla_{\mathcal{H}_Q}\log\left(\frac{d\Pi(\cdot|X^{n})}{d\gamma_{\widehat{\theta}^{(n)}}}\right) 
  \Big\|_{\mathcal{H}_{Q}}^{2}\Bigg] 
  \\& + 2\mathbb{E}_{\Pi(\cdot|X^{n})}\Bigg[\Big\|\nabla_{\mathcal{H}_Q}\log\left(\frac{d\gamma_{\widehat\theta^{(n)}}}{d\widehat{\gamma}_{\widehat\theta^{(n)}}}\right) 
  \Big\|_{\mathcal{H}_{Q}}^{2}\Bigg].    
\end{align*}
Since the gradient $\nabla_{Q}$ of the log-Radon-Nikodym of $\gamma_{\widehat{\theta}^{(n)}}:=\mathcal{N}\left(\widehat{\theta}^{(n)},\left(Q^{-1} + n H^{*}\right)^{-1}\right)$ with respect to $\widehat{\gamma}_{\widehat{\theta}^{(n)}}:=\mathcal{N}\left(\widehat{\theta}^{(n)},\left(Q^{-1} + n \widehat{H}_{n}\right)^{-1}\right)$ is 
\begin{align}
\nabla_{\mathcal{H}_Q}\log\frac{d\gamma_{\widehat{\theta}^{(n)}}}{d\widehat{\gamma}_{\widehat{\theta}^{(n)}}}=nQ(\widehat{H}_{n}-H^{*})(\theta-\widehat{\theta}^{(n)}),
\end{align}
applying the same arguments used in Theorem \ref{bernstein von mises}, we can bound the $KL$ in (\ref{EmpiricalLSI}) by (\ref{LaplaceRate}) evaluated in $\widehat{H}^{(n)}$ (instead of in $H^{*}$) plus the term
\begin{align}
&\frac{1}{2n\lambda_{\min}(\widehat{H}_{n})}\mathbb{E}_{\Pi(\cdot|X^{n})}\Bigg[n^{2}\left\|Q(\widehat{H}_{n}-H^{*})(\theta-\widehat{\theta}^{(n)})\right\|_{\mathcal{H}_{Q}}^{2}\Bigg]
\\ & = \frac{1}{2n\lambda_{\min}(\widehat{H}_{n})}\mathbb{E}_{\Pi(\cdot|X^{n})}\Bigg[n^{2}\left\|Q^{1/2}(\widehat{H}_{n}-H^{*})(\theta-\widehat{\theta}^{(n)})\right\|_{\mathcal{H}}^{2}\Bigg].
\end{align}
Note that we  can bound the term inside the expectation by 
\begin{align}
    n^{2}\left\|Q^{1/2}(\widehat{H}_{n}-H^{*})\right\|_{\text{op}}^{2}\left\|\theta-\widehat{\theta}^{(n)}\right\|_{\mathcal{H}}^{2}.
\end{align}
Then, since $\widehat{H}_{n}-H^{*}={F}_{n}(\widehat{\theta}^{(n)})-{F}(\widehat{\theta}^{(n)})+{F}(\widehat{\theta}^{(n)})-F(\theta^{*})$, 
applying \textbf{(BvM.1)} and \textbf{(BvM.2)} and Young's inequality, there exists a set $\Omega_{1}\subset \mathcal{H}$ of measure at least $1-\delta$ in which 
\begin{align}
&n^{2}\left\|Q^{1/2}(\widehat{H}_{n}-H^{*})\right\|_{\text{op}}^{2}\left\|\theta-\widehat{\theta}^{(n)}\right\|_{\mathcal{H}}^{2} 
\\ &\leq 4n^{2}\left(\left\|\theta^{*}-\widehat{\theta}^{(n)}\right\|^{2}\left(A^{2}+\varepsilon^{(2)}_{1}(n,\delta)^{2}\right) + \varepsilon^{(2)}_{2}(n,\delta)^{2} \right)\left\|\theta-\widehat{\theta}^{(n)}\right\|_{\mathcal{H}}^{2}.
\end{align}
Note now that conditions \textbf{(E.1)} and \textbf{(E.2)} imply that 
\begin{align}\label{newE1}
\mathbb{P}^{n}_{\theta^{*}}\Bigg(\mathbb{E}_{\Pi(\cdot| X^{n})}\Big[\left\|\theta-\widehat{\theta}^{(n)}\right\|_{\mathcal{H}}\Big]\leq \frac{ 2\sigma}{n^{\alpha}}\Bigg) \geq 1-2\delta.
\end{align}
Taking the expectation of the RHS with respect to the posterior and applying (\ref{newE1}) and \textbf{(E.2)}, we will have 
\begin{align*}
 \frac{1}{2n\lambda_{\min}(\widehat{H}_{n})}\mathbb{E}_{\Pi(\cdot|X^{n})}&\left[4n^{2}\left(\left\|\theta^{*}-\widehat{\theta}^{(n)}\right\|^{2}\left(A^{2}+\varepsilon^{(2)}_{1}(n,\delta)^{2}\right) + \varepsilon^{(2)}_{2}(n,\delta)^{2} \right)\left\|\theta-\widehat{\theta}^{(n)}\right\|_{\mathcal{H}}^{2}\right] 
\\ &\leq 
\frac{4n}{\lambda_{\min}(\widehat{H}_{n})}\left(\frac{\sigma^{4}}{n^{4\alpha}}\left(A^{2}+\varepsilon^{(2)}_{1}(n,\delta)^{2}\right) + \varepsilon^{(2)}_{2}(n,\delta)^{2} \frac{\sigma^{2}}{n^{2\alpha}}\right)
\\ &= \frac{4\sigma^{2}}{\lambda_{\min}(\widehat{H}_{n})}\left[\sigma^{2}\left(A^{2}+\varepsilon^{(2)}_{1}(n,\delta)^{2}\right)n^{1-4\alpha}+\varepsilon^{(2)}_{2}(n,\delta)^{2}n^{1-2\alpha}\right].
\end{align*}
We put  
\begin{align*}
\frac{4\sigma^{4}\varepsilon^{(2)}_{1}(n,\delta)^{2}n^{1-2\alpha}}{\lambda_{\min}(\widehat{H}_{n})}
\end{align*}
in the generic term $\mathcal{O}(1/n)$, add the remaining terms to $\mathcal{H}(n,\alpha,\delta,\widehat{H}^{(n)})$, the term analogous to (\ref{LaplaceRate}),  and, denoting by $\mathcal{K}$ the joint term, we will have 
\begin{align}
\mathcal{K}(n,\alpha,\delta,\widehat{H}^{(n)})=&\dfrac{(c_{1}\operatorname{tr}(Q)^{2}+4\sigma^{4})A^{2}}{\lambda_{\min}(\widehat{H}^{(n)})}n^{1-4\alpha} +\dfrac{(c_{2}\operatorname{tr}(Q)+4\sigma^{2})}{\lambda_{\min}(\widehat{H}^{(n)})}\varepsilon_{2}^{(2)}(n,\delta)^{2}n^{1-2\alpha},
\end{align}
and the proof holds.
\end{proof}

\subsection*{Proof of Corollary \ref{example}}

\begin{proof}
Note that the preconditioned Fréchet gradients of $F$ and  $F_{n}$ have the form
\begin{align*}
Q\nabla_{\mathcal H}F(\theta)
&= -QA(\theta-\theta^\ast), \\
\nabla_{\mathcal H}F_n(\theta)
&= -\,QA\theta + QG^\ast\Gamma^{-1}\bar X_n
= -\,QA(\theta-\theta^\ast) + \frac{1}{\sqrt n}\,QG^\ast\Gamma^{-1}\xi, 
\end{align*}
while their preconditioned Hessian coincides and are equal to $-QA$ (as an endomorphism in $\mathcal{H}$). Now, the fluctuation is independent of $\theta$, since:
\begin{equation*}\label{eq:gradfluct}
\nabla_{\mathcal H}F_n(\theta)-\nabla_{\mathcal H}F(\theta)
= \frac{1}{\sqrt n}\,G^\ast\Gamma^{-1}\xi,
\qquad \forall \theta\in\mathcal H.
\end{equation*}
If we precondition the above expression by $Q$, we get 
\begin{equation*}\label{eq:Qgradfluct}
Q\bigl(\nabla_{\mathcal H}F_n(\theta)-\nabla_{\mathcal H}F(\theta)\bigr)
= \frac{1}{\sqrt n}\,QG^\ast\Gamma^{-1}\xi,
\end{equation*}
so
\begin{align*}
\|Q(\nabla_{\mathcal H}F(\theta_1)-\nabla_{\mathcal H}F(\theta_2))\|_{\mathcal H}
&= \|QA(\theta_1-\theta_2)\|_{\mathcal H}
\le \|QA\|_{\mathrm{op}}\,\|\theta_1-\theta_2\|_{\mathcal H}. 
\end{align*}
So $L_1 := \|QA\|_{\mathrm{op}} = \|QG^\ast\Gamma^{-1}G\|_{\mathrm{op}}.$
Now, since the operator $QA$ is coercive, putting $h:=\theta-\theta^\ast$, so
\begin{align*}
-\langle Q\nabla_{\mathcal H}F(\theta),\theta^\ast-\theta\rangle_{\mathcal H}
= \langle QA h, h\rangle_{\mathcal H}, 
\end{align*}
we will have that $\mu = c$. \\ 
\\
Now, from the explicit expression for $Q\nabla_{\mathcal H}F_n(\theta)-Q\nabla_{\mathcal H}F(\theta)$ above, we will have 

\begin{align*}\label{eq:sup-trivial}
\sup_{\theta\in\mathbb B_{\mathcal H}(\theta^\ast,r)}
\|Q\nabla_{\mathcal H}F_n(\theta)-Q\nabla_{\mathcal H}F(\theta)\|_{\mathcal H}
&= \frac{1}{\sqrt n}\,\|QG^\ast\Gamma^{-1}\xi\|_{\mathcal H}.
\end{align*}

Taking $\varepsilon_1(n,\delta)=0$, we can get a $(1-\delta)$-quantile upper bound for $\|QG^\ast\Gamma^{-1}\xi\|_{\mathcal H}$, denoted by $q(\delta)$, and choose $\varepsilon_2(n,\delta)=\frac{1}{\sqrt n}\,q(\delta)$. In order to obtain the bound, it is enough to see that 
$Z:=QG^\ast\Gamma^{-1}\xi$
is a centered Gaussian random variable in $\mathcal H$, whose covariance operator is
\begin{equation*}\label{eq:covZ}
\text{Cov}(Z)
= QG^\ast\Gamma^{-1}\text{Cov}(\xi)\,\Gamma^{-1}GQ
= QAQ.
\end{equation*}
Now, since $QAQ$ is trace-class, a standard Gaussian concentration bound gives us 
\begin{equation*}\label{eq:gauss-conc}
\mathbb P\!\left(
\|Z\|_{\mathcal H}
\le \sqrt{\mathrm{tr}(QAQ)} + \sqrt{2\|QAQ\|_{\mathrm{op}}\log(1/\delta)}
\right)\ \ge\ 1-\delta \quad \forall \delta\in(0,1),
\end{equation*}
and the result follows.

\end{proof}

\section{Background on Fréchet and Malliavin derivatives}
In this section we present the definitions of the  derivatives and gradients used in the paper. We follow \cite{malliavin2} and \cite{malliavin1}.
\\[4pt]
Let $(\mathcal H,\langle\cdot,\cdot\rangle_{\mathcal H})$ be a separable real Hilbert space and let
$\mathcal H_0\subset \mathcal H$ be a Hilbert subspace continuously embedded into $\mathcal H$, i.e.
there exists $C>0$ such that $\|h\|_{\mathcal H}\le C\|h\|_{\mathcal H_0}$ for all $h\in\mathcal H_0$.
Let $U\subset \mathcal H$ be open and let $f:U\to\mathbb R$.

\begin{definition}[Fr\'echet derivative on $\mathcal H$]
We say that $f$ is \emph{Fr\'echet differentiable at $\theta\in U$ (as a map on $\mathcal H$)} if there exists
a bounded linear functional $D_{\mathcal{H}}f(\theta)\in \mathcal H^\ast$ such that
\begin{align}
\lim_{\|h\|_{\mathcal H}\to 0}
\frac{\big|f(\theta+h)-f(\theta)-Df(\theta)[h]\big|}{\|h\|_{\mathcal H}}=0.
\end{align}
When this holds, by the Riesz representation theorem there exists a unique element
$\nabla_{\mathcal H}f(\theta)\in \mathcal H$ such that
\begin{align}
D_{\mathcal{H}}f(\theta)[h]=\langle \nabla_{\mathcal H}f(\theta),h\rangle_{\mathcal H},\quad \forall h\in\mathcal H.
\end{align}
We call $\nabla_{\mathcal H}f(\theta)$ the \emph{$\mathcal H$-gradient} of $f$ at $\theta$.
\end{definition}

\begin{definition}[Fréchet derivative along $\mathcal H_0$]
We say that $f$ is \emph{Fr\'echet differentiable at $\theta\in U$ along $\mathcal H_0$} if there exists
a bounded linear functional $D_{\mathcal H_0}f(\theta)\in \mathcal H_0^\ast$ such that
\begin{align}
\lim_{\|h\|_{\mathcal H_0}\to 0}
\frac{\big|f(\theta+h)-f(\theta)-D_{\mathcal H_0}f(\theta)[h]\big|}{\|h\|_{\mathcal H_0}}=0,
\qquad h\in \mathcal H_0,
\end{align}
whenever $\theta+h\in U$. In this case there exists a unique element
$\nabla_{\mathcal H_0}f(\theta)\in \mathcal H_0$ satisfying
\begin{align}
D_{\mathcal H_0}f(\theta)[h]=\langle \nabla_{\mathcal H_0}f(\theta),h\rangle_{\mathcal H_0},
\qquad \forall h\in \mathcal H_0,
\end{align}
and we call $\nabla_{\mathcal H_0}f(\theta)$ the \emph{$\mathcal H_0$-gradient} of $f$ at $\theta$.
\end{definition}

\begin{remark}[Relation between $Df$ and $D_{\mathcal H_0}f$]
If $f$ is Fr\'echet differentiable at $\theta$ on $\mathcal H$, then it is Fr\'echet differentiable at $\theta$
along $\mathcal H_0$ and
\[
D_{\mathcal H_0}f(\theta)=Df(\theta)\circ i,
\]
where $i:\mathcal H_0\hookrightarrow \mathcal H$ is the canonical embedding.
Let $i^\ast:\mathcal H\to \mathcal H_0$ be the adjoint operator characterized by
$\langle i^\ast x,h\rangle_{\mathcal H_0}=\langle x,ih\rangle_{\mathcal H}$ for all $x\in\mathcal H$, $h\in\mathcal H_0$.
Then, whenever $\nabla_{\mathcal H}f(\theta)$ exists, we have
\[
\nabla_{\mathcal H_0}f(\theta)=i^\ast \nabla_{\mathcal H}f(\theta).
\]
\end{remark}

Assume now that $\mu_0$ is a centered Gaussian measure on $\mathcal H$ with Cameron--Martin space $\mathcal H_0$.
Let $W=\{W(h):h\in\mathcal H_0\}$ be the associated isonormal Gaussian process, i.e.
$\mathbb E[W(h)W(g)]=\langle h,g\rangle_{\mathcal H_0}$.
A \emph{smooth cylindrical functional} is a random variable of the form
\[
F=\varphi\big(W(h_1),\dots,W(h_m)\big),
\qquad h_1,\dots,h_m\in\mathcal H_0,\ \ \varphi\in C_b^\infty(\mathbb R^m).
\]
The \emph{Malliavin derivative} (or Gross--Sobolev derivative) of such $F$ is the $\mathcal H_0$-valued random variable
\[
DF=\sum_{j=1}^m \partial_j\varphi\big(W(h_1),\dots,W(h_m)\big)\,h_j \in \mathcal H_0.
\]
Equivalently, for any direction $k\in\mathcal H_0$,
\begin{equation}\label{eq:dir-mall}
D_kF:=\langle DF,k\rangle_{\mathcal H_0}
=\left.\frac{d}{d\varepsilon}\right|_{\varepsilon=0}
\varphi\big(W(h_1)+\varepsilon\langle h_1,k\rangle_{\mathcal H_0},\dots,W(h_m)+\varepsilon\langle h_m,k\rangle_{\mathcal H_0}\big).
\end{equation}
Thus $DF$ is the Riesz representer (in $\mathcal H_0$) of the directional derivative $k\mapsto D_{k}F$. More generally, if $G:\mathcal H\to\mathbb R$ is Fr\'echet differentiable along $\mathcal H_0$ and $X\sim \mu_0$,
then, under the integrability condition $\nabla_{\mathcal H_0}G(X)\in L^2(\mu_0;\mathcal H_0)$, the chain rule yields
\begin{equation}\label{eq:chain}
D\big(G(X)\big)=\nabla_{\mathcal H_0}G(X)\qquad\text{in }L^2(\Omega;\mathcal H_0).
\end{equation}
In particular, if $\mathcal H_0=\mathcal H_Q$ is the Cameron--Martin space of $\mu_0=\mathcal N(0,Q)$, and if
$G$ is Fr\'echet differentiable on $\mathcal H$ with gradient $\nabla_{\mathcal H}G$, then
$\nabla_{\mathcal H_Q}G(x)=Q\,\nabla_{\mathcal H}G(x)$ (on the natural domain), and hence
\[
D\big(G(X)\big)=\nabla_{\mathcal H_Q}G(X)=Q\,\nabla_{\mathcal H}G(X).
\]

Finally, we recall that the operator $D$ defined on smooth cylindrical functionals is closable in $L^2(\mu_0)$.
Its closure defines the Sobolev space $\mathbb D^{1,2}$, equipped with the norm
\[
\|F\|_{\mathbb D^{1,2}}^2 := \mathbb E_{\mu_0}\big[F^2\big]+\mathbb E_{\mu_0}\big[\|DF\|_{\mathcal H_0}^2\big].
\]
When $V:\mathcal H\to\mathbb R$ satisfies $V\in \mathbb D^{1,2}$, we write $DV$ for its Malliavin derivative,
and the condition $DV\in L^2(\mu_0;\mathcal H_0)$ is equivalent to $\nabla_{\mathcal H_0}V\in L^2(\mu_0;\mathcal H_0)$.

\section{Background on stochastic integration in Hilbert spaces}
In this section we introduce the necessary background about stochastic calculus in infinite dimensional Hilbert Spaces. We follow mainly the exposition in \cite{liu2015stochastic}.
\begin{definition}\textbf{(Hilbert Space-valued Martingale, \cite{liu2015stochastic})}
    \item Let $\{M_t\}_{t\ge 0}$, where 
$$
    M_t:(\Omega,F,P)\longrightarrow \mathcal{H},\quad \text{with } \mathcal{H} \text{ a separable Hilbert Space}.
$$
  It is required that $\{\mathcal{F}_t\}_{t\ge 0}$ be a filtration on $(\Omega,F,P)$
  and that the following conditions hold:
\end{definition}

\begin{itemize}
  \item $M_t$ is $\mathcal{F}_t$-measurable for every~$t$.
  \item $M_t$ is Bochner-integrable; that is, for every $t\ge 0$ we have
$$\mathbb E\left[\|M_t\|_H\right]<\infty,$$
  which is equivalent to
$$\int_{\Omega}\|M_t(\omega)\|_H\,d\mu_t(\omega)<\infty.$$
\end{itemize}

Furthermore, we assume that
$$
  \mathbb E\left[M_t \mid \mathcal{F}_s\right]=M_s
  \quad\text{almost surely,}\quad \forall\,s\in[0,t],\;t\ge 0.
$$

\begin{theorem}
    \textbf{(Doob’s maximal inequality,  \cite{liu2015stochastic})}:  
Let $\{M_t\}_{t\in[0,T]}$ be a right-continuous $\mathcal{H}$-martingale. Then
$$
\Bigl(\E\Bigl[\sup_{t\in[0,T]}\|M_t\|_H^{p}\Bigr]\Bigr)^{\!1/p}
\le \frac{p}{p-1}\,
      \sup_{t\in[0,T]}\Bigl(\E\!\left[\|M_t\|^p\right]\Bigr)^{\!1/p}
  = \frac{p}{p-1}\,
      \E\!\left[\|M_{T}\|^p\right]^{1/p}.$$
\end{theorem}

\begin{theorem}
 \textbf{(Integrable martingale spaces, \cite{liu2015stochastic})}: Define

$$\|M\|_{M_T}^2=\sup_{t\in[0,T]}
    \Bigl(\E\!\left[\|M_t\|^2\right]\Bigr)^{\!1/2}.$$
Then, we have the equivalency
$$\E\!\left[\|M_T\|^2\right]^{1/2}
  \le\E\!\left[\sup_{t\in[0,T]}\|M_t\|^2\right]^{1/2}
  \le 2\,\E\!\left[\|M_{T}\|^2\right]^{1/2}.$$
\end{theorem}

\begin{theorem}\label{theorem:real-valued-martingale}\textbf{(Real-valued martingales as $1$-dimensional projections of $\mathcal{H}$-valued martingales, \cite{liu2015stochastic})} 
\\ \\
Given $\Phi\in N_W^2(0,T)$, the process
$$\int_0^t\langle\Phi(s),dW^{Q}(s)\rangle_H
$$
is a real-valued martingale for every $t\in[0,T]$, and
\[
  \sup_{t\in[0,T]}
      \mathbb{E}\left[\left\|\int_{0}^{t}\langle\Phi(s),dW^{Q}(s)\rangle_H\right\|^2\right]
  =\mathbb{E}\left[\left\|\int_{0}^{T}\langle\Phi(s),dW^{Q}(s)\rangle_H\right\|^2\right].
\]    
\end{theorem}

\begin{theorem}\label{theorem:ito-formula}
\textbf{(Itô’s Formula in Hilbert Space)}
Let $F(t,X(t)):[0,T]\times \mathcal{H} \rightarrow \mathbb{R}$ be a sufficiently differentiable function.  
There exists a null set $N\in F$ such that, on its complement $N^{c}$ and for all $t\in[0,T]$, we have
\begin{align*}
& F(t,X(t)) = F(0,X(0))
  + \int_0^t\langle \nabla_{\mathcal{H}}F(s,X(s)),\,\sigma(s)\,dW(s)\rangle_{\mathcal{H}}
  \\
&  +\int_0^t\Bigl(\tfrac{\partial F}{\partial t}(s,X(s)) + \langle \nabla_{\mathcal{H}} 
F(s,X(s)),\,\mu(s)\rangle_{\mathcal{H}}ds \\ &+ \frac{1}{2}\operatorname{tr}\Bigl[(\sigma(s)Q^{1/2})\,D^{2}_{\mathcal{H}}F(s,X(s))\,(\sigma(s)Q^{1/2})\Bigr]\Bigr)\,ds.
\end{align*}

\end{theorem}

\section{Technical results about Gaussian Measures}
We present here some now classical results on the construction of infinite-dimensional probability measures from the centered Gaussian reference measure with covariance operator $Q$, $\mu:=\mathcal{N}_{0,Q}$. We denote by $L_{1}^{+}(\mathcal{H})$ the space of linear positive definite trace-class operators over $\mathcal{H}$. The results are taken from \cite{daprato}.  

\begin{theorem}[\textbf{Fernique, \cite{daprato}}]\label{thm:fernique}
There exists $\alpha>0$ such that
\[
  \int_H \exp\bigl[\alpha\,\|x\|_\mathcal{H}^2\bigr]\,d\mu(x)\;<\;\infty.
\]
\end{theorem}

\begin{theorem}(\textbf{Cameron-Martin, \cite{daprato}})
Let $\mu=N_{0,Q}$ and $\nu=N_{a, Q}$ on $(\mathcal{H}, \mathscr{B}(\mathcal{H}))$, where $a \in \mathcal{H}$ and $Q \in L_{1}^{+}(\mathcal{H})$.
Then, 
\begin{description}
    \item If $a \notin \mathcal{H}_{Q}$ then $\mu$ and $\nu$ are singular.\\
 \item If $a \in \mathcal{H}_{Q}$ then $\mu$ and $\nu$ are equivalent.\\
 \item If $\mu$ and $\nu$ are equivalent the density $\frac{d\nu}{d \mu}$ is given by
\end{description}
\begin{equation*}
\frac{d\nu}{d \mu}(\theta)=\exp \left\{-\frac{1}{2}\|a\|^{2}_{\mathcal{H}_{Q}}+\langle a,\theta \rangle_{\mathcal{H}_{Q}}\right\}, \quad \theta \in \mathcal{H}.
\end{equation*}
\end{theorem}

\begin{theorem}(\textbf{Feldman-Hajék, \cite{daprato}})
Let $Q, R \in L_{1}^{+}(\mathcal{H})$ be such that $[Q, R]:=Q R-R Q=0$. Let $\mu=N_{Q}$ and $\nu=N_{R}$. Then $\mu$ and $\nu$ are equivalent if and only if
\begin{align}
\sum_{k=1}^{\infty} \frac{\left(\lambda_{k}-r_{k}\right)^{2}}{\left(\lambda_{k}+r_{k}\right)^{2}}<\infty   
\end{align}

If $\mu$ and $\nu$ are not equivalent they are singular.
\end{theorem}
Note that this condition is equivalent to to having 
\begin{align*}
\sum_{k=1}^{\infty} \left(\frac{\lambda_{k}}{r_{k}}-1\right)^{2}<\infty, \qquad \lambda_{k}= 0 \iff r_{k}= 0.
\end{align*}

\end{document}